\documentclass[10pt,journal,compsoc]{IEEEtran}
\ifCLASSOPTIONcompsoc
  \usepackage[nocompress]{cite}
\else
  \usepackage{cite}
\fi
\ifCLASSINFOpdf
\else
\fi
\usepackage{amsthm} 
\usepackage{amsmath}
\newtheorem{theorem}{Theorem}
\newtheorem{example}{Example}
\newtheorem{definition}{Definition}

\usepackage{adjustbox}
\usepackage{cases}
\usepackage[overload]{empheq}
\usepackage{subfig}
\usepackage{graphicx}
\usepackage{mathrsfs}
\usepackage[lined,boxed,commentsnumbered,ruled, linesnumbered]{algorithm2e}
\usepackage{algorithmic}
\usepackage{color}
\usepackage{makecell}
\usepackage{bm}
\usepackage{url}
\usepackage{hyperref}
\usepackage{mathtools}
\usepackage{balance}
\usepackage{nopageno}
\usepackage{xcolor}
\usepackage{color, soul}
\usepackage{multirow}
\usepackage{bbm}
\usepackage[normalem]{ulem}
\usepackage{amssymb} 

\newcommand{\nop}[1]{}

\begin{document}
%
\title{FairSample: Training Fair and Accurate Graph Convolutional Neural Networks Efficiently}

%
%
%
%


\author{
Zicun Cong, Baoxu Shi, Shan Li, Jaewon Yang, Qi He,~\IEEEmembership{Senior Member, IEEE}, Jian Pei,~\IEEEmembership{Fellow, IEEE}
\IEEEcompsocitemizethanks{
\IEEEcompsocthanksitem Zicun Cong was with School of Computing Science, Simon Fraser University. E-mail: zcong@sfu.ca\protect\\
\vspace{-0.20cm}
\IEEEcompsocthanksitem Baoxu Shi was with LinkedIn Corporation, Sunnnyvale, USA. E-mail: dashi@linkedin.com\protect\\
\vspace{-0.20cm}
\IEEEcompsocthanksitem Shan Li was with LinkedIn Corporation, Sunnnyvale, USA. E-mail: shali@linkedin.com\protect\\
\vspace{-0.20cm}
\IEEEcompsocthanksitem Jaewon Yang was with LinkedIn Corporation, Sunnnyvale, USA. E-mail: jeyang@linkedin.com\protect\\
\vspace{-0.20cm}
\IEEEcompsocthanksitem Qi He was with LinkedIn Corporation, Sunnnyvale, USA. E-mail: qhe@linkedin.com\protect\\
\vspace{-0.20cm}
\IEEEcompsocthanksitem Jian Pei is with School of Computing Science, Simon Fraser University. E-mail: jpei@cs.sfu.ca\protect\\
}
}

%
%

\markboth{Journal of \LaTeX\ Class Files,~Vol.~14, No.~8, August~2015}%
{Shell \MakeLowercase{\textit{et al.}}: Bare Demo of IEEEtran.cls for Computer Society Journals}
%



\IEEEtitleabstractindextext{%
\begin{abstract}
Fairness in Graph Convolutional Neural Networks (GCNs) becomes a more and more important concern as GCNs are adopted in many crucial applications. Societal biases against sensitive groups may exist in many real world graphs. GCNs trained on those graphs may be vulnerable to being affected by such biases.
In this paper, we adopt the well-known fairness notion of demographic parity and tackle the challenge of training fair and accurate GCNs efficiently.
We present an in-depth analysis on how graph structure bias, node attribute bias, and model parameters may affect the demographic parity of GCNs. 
Our insights lead to FairSample, a framework that jointly mitigates the three types of biases. We employ two intuitive strategies to rectify graph structures. First, we inject edges across nodes that are in different sensitive groups but similar in node features. Second, to enhance model fairness and retain model quality, we develop a learnable neighbor sampling policy using reinforcement learning. To address the bias in node features and model parameters, FairSample is complemented by a regularization objective to optimize fairness.
\end{abstract}

\begin{IEEEkeywords}
Graph neural network, Sampling, Fairness
\end{IEEEkeywords}}

\maketitle

\IEEEdisplaynontitleabstractindextext

%
\IEEEpeerreviewmaketitle

\IEEEraisesectionheading{\section{Introduction}\label{sec:introduction}}

\label{sec:intro}

Graph Convolutional Neural Networks (GCNs) are deep neural networks learning latent representations for a given graph~\cite{huangAdaptiveSamplingFast2018, DBLP:conf/iclr/KipfW17}. GCNs learn latent representations of nodes (node embeddings) from the graph topology and then use the embeddings for prediction tasks.
Due to the superior prediction performance, GCNs enjoy many real world applications, such as talent recruitment~\cite{DBLP:conf/ijcnn/LiuZLSH21}, healthcare~\cite{DBLP:conf/kdd/CuiSTMWL20}, and  recommender system~\cite{DBLP:conf/sigir/0001DWLZ020}. At the same time, fairness in GCNs becomes a serious concern. GCNs can inherit the bias in the underlying graphs and make discriminatory decisions correlated with sensitive attributes, such as age, gender, and nationality~\cite{DBLP:conf/wsdm/DaiW21, DBLP:conf/aistats/LaclauRCL21, DBLP:conf/iclr/LiWZHL21}. 
In order to use GCNs for many real world applications, such as job recommendation, it is critical for GCNs to make fair predictions for people with different sensitive  values~\cite{DBLP:journals/csur/MehrabiMSLG21}.

Demographic parity~\cite{DBLP:conf/wsdm/DaiW21, DBLP:conf/icml/AgarwalBD0W18} is a well adopted notion of fairness, which aims for similar predictions for different demographic groups (see Section~\ref{sec:preliminary_fairness} for a brief review). Being closely related to the legal doctrine of disparate impact~\cite{DBLP:conf/ijcai/RahmanS0019, DBLP:journals/corr/abs-2201-00292}, demographic parity has drawn tremendous attention.
There are many challenges in improving the demographic parity fairness of GCNs. First, GCNs leverage node features to generate embeddings, but node features are often correlated with sensitive attributes~\cite{DBLP:conf/wsdm/DaiW21}. For example, it has been shown that a user's postal code may manifest the user's nationality~\cite{barocas-hardt-narayanan}. 
Second, a graph structure can also inherit societal bias -- nodes with the same sensitive values are more likely to be connected due to homophily. 
For example, it has been observed in social networks that people with the same religions are more likely to become friends~\cite{thelwall2009homophily}. 
GCNs learn node embeddings by aggregating
and transforming node feature vectors over graph
topology. 
In a biased graph, a node mainly receives information from neighbors in the same sensitive group. Therefore, the embeddings learned on the graph tend to be correlated with sensitive attributes and lead to discriminatory predictions~\cite{DBLP:conf/iclr/LiWZHL21, DBLP:conf/www/DongLJL22}.
Last, there is a well-known tradeoff between accuracy and demographic parity fairness of machine learning models~\cite{DBLP:conf/uai/AgarwalLZ21, DBLP:conf/fat/MenonW18, DBLP:journals/jmlr/ZhaoG22}. A fair GCN should achieve the desirable fairness level with less sacrifice in utility~\cite{DBLP:conf/iclr/LiWZHL21}.

There are some pioneering methods tackling demographic parity fairness in GCNs.  For example, FairGNN~\cite{DBLP:conf/wsdm/DaiW21} learns a demographic parity fair GCN classifier by adding one or more fairness penalty terms in the loss function of the classifier, which rectifies the model parameters. However, FairGNN ignores the important feature aggregation mechanisms of GCNs and fails to explore their effects on improving model fairness. As a result, the fairness and accuracy performance achieved by FairGNN is limited.

A few unsupervised fairness enhancing methods~\cite{DBLP:conf/iclr/LiWZHL21, DBLP:conf/www/DongLJL22} optimize the graph adjacency matrix by dropping edges so that the learned node embeddings are independent of sensitive attributes. As they adjust the adjacency matrix by assigning a learnable parameter to each edge of the input graph, they cannot scale up to large real world networks. Moreover, as the node embeddings in those methods are not optimized for specific graph mining tasks, the learned embeddings may not perform well in accuracy and fairness on a user's target task, such as node classification. Last, dropping edges may not be enough to adjust the biased neighborhood composition if all neighbors of a node have the same sensitive value.

In this paper, we tackle the challenges of training demographic parity fair and accurate GCN models efficiently and make several contributions. First, we systematically analyze how graph structure bias, node attribute bias, and model parameters may affect the demographic parity of GCNs. Second, based on the analysis and insights, we propose FairSample, a novel framework to improve the demographic parity of GCNs with a minimal tradeoff in accuracy and efficiency.  FairSample has three major components to address the challenges in building fair and accurate GCNs efficiently. To address the societal bias in an input graph, FairSample has a novel edge injector and an original computation graph sampler to rectify the feature aggregation process of GCNs. The edge injector adds edges across nodes with different sensitive values and ensures that each node has enough neighbors with different sensitive values from which GCNs can learn unbiased embeddings. Computation graph sampler selects neighbors that help GCNs make both accurate and fair predictions. We employ reinforcement learning to train the sampler with a small number of parameters such that it enhances the demographic parity of GCNs and retains good accuracy and scalability. To address the bias in node features and model parameters, FairSample is complemented by a regularization objective to optimize fairness following the state-of-the-art approach~\cite{DBLP:conf/wsdm/DaiW21}.
Last, we conduct extensive experiments on a series of real world datasets to demonstrate that FairSample has superior capability in training fair and accurate GCNs. Our results show that FairSample outperforms the state-of-the-art fair GCN methods in terms of both fairness and accuracy. Compared to the GCN methods without fairness constraints, FairSample achieves up to $65.5\%$ improvement in fairness at the cost of only at most $5.0\%$ accuracy decrease. Furthermore, FairSample incurs only a small extra computational overhead, making it more computationally efficient than the majority of the baseline methods.

The rest of the paper is organized as follows. In Section~\ref{sec:related_work}, we review the related work.  The essentials of graph convolutional neural networks and demographic parity are reviewed in Section~\ref{sec:preliminaries}. In Section~\ref{sec:problem_formulation}, we explore the intuitions and formulate the problem.  We develop FairSample in Section~\ref{sec:proposed_method} and report the experimental results in Section~\ref{sec:experiments}.  In Section~\ref{sec:conclusion}, we conclude the paper.

\section{Related Work}\label{sec:related_work}

In this section, we briefly review the related work on fair GCNs~\cite{DBLP:conf/icml/BoseH19, DBLP:conf/wsdm/DaiW21, DBLP:conf/iclr/LiWZHL21} and sampling strategies for efficient GCN training~\cite{DBLP:conf/nips/HamiltonYL17, chenFastgcnFastLearning2018, DBLP:conf/kdd/YoonGSNHY21}.

\subsection{Fair Graph Convolutional Neural Networks}

Most existing fair machine learning methods are developed for independently and identically distributed (i.i.d) data~\cite{DBLP:conf/iclr/LiWZHL21}. They can be categorized into three groups.
The pre-processing methods enhance model fairness by repairing biased training datasets, such as revising data labels~\cite{DBLP:conf/icdm/KamiranCP10, DBLP:journals/kais/KamiranC11, DBLP:conf/kdd/ThanhRT11} and perturbing data attributes~\cite{DBLP:journals/tkde/HajianD13, DBLP:conf/aistats/KilbertusRSMV20, DBLP:conf/kdd/FeldmanFMSV15}. 
The post-processing methods enhance the fairness of a model by calibrating the prediction scores of the model~\cite{DBLP:conf/fat/MenonW18, DBLP:conf/icdm/KamiranKZ12, DBLP:conf/nips/ValeraSR18}. 
The in-processing methods revise the training process of specific machine learning models to enhance model fairness. Typically, fairness notion related regularizers are added into model objective functions~\cite{DBLP:conf/nips/DoniniOBSP18, DBLP:conf/ijcai/ManishaG20, DBLP:conf/icml/AgarwalBD0W18}. 
Please see~\cite{DBLP:journals/csur/MehrabiMSLG21} for a survey of fair machine learning models.

In general, graph data is non-IID.
It has been discovered that both node attributes and graph structure can contain societal bias~\cite{DBLP:conf/www/DongLJL22, DBLP:conf/wsdm/DaiW21}.
Due to the feature aggregation mechanisms of GCNs, the predictions of GCNs may inherit both types of biases~\cite{DBLP:conf/wsdm/DaiW21}. However, 
the existing fairness methods on i.i.d data are not suitable to be directly applied to GCNs, as they cannot handle the graph structure bias~\cite{DBLP:conf/iclr/LiWZHL21}.
Very recently, there are exciting advances in enforcing fairness in GCNs. Most studies adopt demographic parity as the fairness measure.
FCGE~\cite{DBLP:conf/icml/BoseH19}, FairGNN~\cite{DBLP:conf/wsdm/DaiW21}, and DFGNN~\cite{DBLP:conf/esann/OnetoND20} apply adversarial learning to improve the fairness in node embeddings. The embeddings are optimized to prevent a discriminator from accurately predicting the corresponding sensitive attribute values.
These methods do not explicitly rectify biases in graph structures and have suboptimal fairness performance.

A few pre-processing methods are proposed to remedy graph structure biases and improve demographic parity.
FairWalk~\cite{DBLP:conf/ijcai/RahmanS0019} extends the node2vec graph embedding method~\cite{DBLP:conf/kdd/GroverL16} using sensitive attribute-aware stratified sampling to create a diversified node context. Li~\textit{et~al.}~\cite{DBLP:conf/iclr/LiWZHL21} and Laclau~\textit{et~al.}~\cite{DBLP:conf/aistats/LaclauRCL21} propose to adjust the adjacency matrix of an input graph such that nodes from different sensitive groups have similar embeddings. EDITS~\cite{DBLP:conf/www/DongLJL22} produces a fair version of the input graph by altering its adjacency matrix and node attributes. Kose and Shen~\cite{DBLP:journals/corr/abs-2201-08549} propose a fairness-aware data augmentation framework on node attributes and graph structures for graph contrastive learning. As these methods do not optimize their utility on specific tasks, they may lose considerable accuracy on a user's target task.

Some studies explore other fairness notions in GCNs. For example, Agarwal~\textit{et~al.}~\cite{DBLP:conf/uai/AgarwalLZ21} and Ma~\textit{et~al.}~\cite{DBLP:conf/wsdm/MaGWYZL22} adopt contrastive learning to train GCN classifiers under the constraint of counterfactual fairness~\cite{DBLP:conf/nips/KusnerLRS17}.
In favor of individual fairness, Dong~\textit{et~al.}~\cite{DBLP:conf/kdd/DongKTL21} propose a ranking-based regularizer on GCNs. Some studies~\cite{genfairgnn_neurips21, 10.1145/3340531.3411872, DBLP:journals/corr/abs-2202-13547} investigate and mitigate the accuracy disparity between sensitive groups. These methods, however, may not improve demographic parity in GCNs~\cite{DBLP:journals/csur/MehrabiMSLG21}.

Our approach FairSample is an in-processing fairness enhancing method, which revises the training process of GCNs to enhance their demographic parity. Different from the existing fairness enhancing methods, 
FairSample rectifies the feature aggregation process and models parameters of GCNs at the same time to explicitly optimize GCNs towards good accuracy and fairness.

\subsection{Sampling Strategies for Efficient GCN Training}

Most GCN models~\cite{DBLP:conf/iclr/KipfW17,  DBLP:conf/wsdm/DaiW21} use the whole $K$-hop neighborhood of each node to train node embeddings and result in high computation and memory costs. Thus, those methods cannot scale up to large graphs~\cite{DBLP:conf/nips/HamiltonYL17, DBLP:conf/kdd/YoonGSNHY21}.
To scale up GCN models, sampling strategies have been proposed to limit the neighborhood size for each given node and reduce the cost of feature aggregation. 
The sampling strategies can be a heuristic probability distribution~\cite{DBLP:conf/nips/HamiltonYL17, chenFastgcnFastLearning2018, huangGraphRecurrentNetworks2019, bojchevskiScalingGraphNeural2020} or a learned parameterized model~\cite{liuBanditSamplersTraining2020,congMinimalVarianceSampling2020,DBLP:conf/kdd/YoonGSNHY21}.

Our method also employs a sampling strategy for training GCNs. Different from the existing  methods, our sampling strategy is designed to sample balanced and informative sets of neighbors to enhance model fairness with minimal tradeoff in model utility. 

\section{Preliminaries}
\label{sec:preliminaries}
In this section, we review the essentials of graph convolutional neural networks (GCNs) and how to train fair GCNs. Table~\ref{tbl:notations} shows the commonly used notations in the paper.

\begin{table}[t]
\centering
\caption{Frequently used notations.}
\begin{tabular}{|c|c|}
\hline
\textbf{Symbol} & \textbf{Definition} \\
\hline
 $\mathcal{G} = (\mathcal{V}, \mathcal{E}, \mathbf{X})$ & \makecell{An attributed graph with nodes $\mathcal{V}$, edges $\mathcal{E}$,\\ and node feature vectors $\mathbf{X}$} \\
 $\mathcal{A} = \{a_1, \ldots, a_\zeta\}$ & Domain of the categorical sensitive attribute \\
 $\mathcal{V}_{a_i}$ & The sets of nodes taking the sensitive value $a_i$\\
 $\mathbf{x}_i$ & Feature vector of node $v_i$ \\
 $s_i \in \mathcal{A}$ & Sensitive attribute value of $v_i$ \\
 $\Gamma_v$ & The set of 1-hop neighbors of node $v$ \\
 \hline
 $f_{G}$ & A $K$-layer GCN node classifier \\
 $\mathbf{h}_i^l\ (1\leq l\leq K)$ & Embedding of $v_i$ in the $l$-th layer of a GCN\\
 $\mathbf{W}_l\ (1\leq l\leq K)$ & Parameters in the $l$-th layer of a GCN\\
 $\mathcal{T}_{v_i}$ & Computation graph of $v_i$ \\
 $\mathcal{L}_{\text{dp}}$ & Demographic parity regularizer \\
 \hline
 $p(v)$ & A sample of child nodes of $v$\\
 $\text{P}(v_j | v_i)$ & Probability of sampling a child node $v_j$ of $v_i$\\
 \hline
\end{tabular}
\label{tbl:notations}
\end{table}

\subsection{Graph Convolutional Neural Networks}
\label{sec:notations}
Denote by $\mathcal{G} = (\mathcal{V}, \mathcal{E}, \mathbf{X})$ an \emph{attributed graph}, where $\mathcal{V}=\{v_1, \ldots, v_n\}$ is a set of $n$ nodes, $\mathcal{E} \subseteq \mathcal{V} \times \mathcal{V}$ is a set of edges, $\mathbf{X} \in \mathbb{R}^{n \times d}$ is a matrix of node features, and $\mathbf{x}_i \in \mathbb{R}^d$ is the feature vector of node $v_i$ $(1 \leq i \leq n)$. 
Denote by $\Gamma_v$ the set of 1-hop neighbors of node $v$. To investigate the fairness in GNNs, we assume that all nodes have a categorical \emph{sensitive attribute} $\mathcal{A}$, such as age, gender, and disability. Each node $v_i$ is associated with a \emph{sensitive value} $s_i \in \mathcal{A}$, where $\mathcal{A} = \{a_1, \ldots, a_\zeta\}$ is the domain of the sensitive attribute. Denote by $\mathcal{V}_{a_i} = \{v_j | v_j \in \mathcal{V}, s_j = a_i\}$ and $\mathcal{V}_{\neq a_i} = \mathcal{V} \setminus \mathcal{V}_{a_i}$ $(1 \leq i \leq \zeta)$ the sets of nodes taking and not taking the sensitive value $a_i$, respectively.

We focus on the semi-supervised node classification problem, which has wide practical applications~\cite{DBLP:conf/wsdm/DaiW21, DBLP:conf/iclr/KipfW17}. Our proposed method can also be applied to other graph learning tasks, like link prediction problem. Denote by $\mathcal{V}_L = \{v_1, \ldots, v_m\} \subseteq \mathcal{V}$ a set of labeled nodes and by $Y_L = \{y_1, \ldots, y_m\}$ the corresponding class labels. We follow the common setting in fair machine learning~\cite{DBLP:conf/wsdm/DaiW21, DBLP:conf/ijcai/ManishaG20} and assume that the class labels $y_i \in \{0, 1\}$ $(1 \leq i \leq m)$ are binary. Our proposed method can also be applied to multi-class classification problems.

A $K$-layer GCN~\cite{DBLP:conf/nips/HamiltonYL17, DBLP:conf/kdd/JiaLYYLA20} uses graph structures and node attributes to predict node labels.
In the $l$-th layer  $(1 \leq l \leq K)$, the GCN updates the embedding $\mathbf{h}^l_v$ of node $v$ by aggregating the $(l-1)$-th layer embeddings of the neighbors of $v$. A simple and general formulation of GCNs~\cite{chenFastgcnFastLearning2018, DBLP:conf/kdd/YoonGSNHY21} is
\begin{equation}
    \mathbf{h}_{v}^{l} = \rho(\frac{1}{|\Gamma_v|+1} \sum_{v_i \in \Gamma_v \cup \{v\}} \mathbf{h}_{v_i}^{l-1} \mathbf{W}_{l}),
    \label{eq:vanilla_emebdding}
\end{equation}
where $\rho(\cdot)$ is a non-linear activation function, $\mathbf{W}_{l} \in \mathbb{R}^{d_{l-1} \times d_{l}}$ is a learnable transformation matrix in the $l$-th layer of a GCN, $d_{l}$ is the dimensionality of the $l$-th layer node embeddings, and $\mathbf{h}_v^{0} = \mathbf{x}_{v}$. The GCN predicts the label of $v$ as $\hat{y}_v = \sigma(\mathbf{h}^{K}_v \mathbf{W}_c)$, where $\mathbf{W}_c$ is the classifier parameters and $\sigma(\cdot)$ is the softmax function. The GCN is trained by optimizing the cross entropy loss $\mathcal{L}_{\text{accuracy}}$.

Given a very large graph, the limited GPU memory is often not enough to store the whole graph. To leverage GPU for fast computation, GCNs are trained using stochastic gradient descent, that is, only mini-batches of nodes and their corresponding computation graphs are loaded into GPU. \textit{Computation graph} describes how features are aggregated and transformed in each GCN layer to compute node embeddings~\cite{DBLP:conf/kdd/JiaLYYLA20}.
For each node $v \in \mathcal{V}$, the computation graph $\mathcal{T}_v$ of $v$ uses a tree structure defining how to compute embedding $\mathbf{h}_{v}^l$ by aggregating the previous layer embeddings of the neighbors of $v$~\cite{DBLP:conf/kdd/JiaLYYLA20}. Given $\mathcal{T}_{v}$, the embedding $\mathbf{h}_{v}^K$ can be computed by iteratively aggregating embeddings along the edges in $\mathcal{T}_{v}$ in a bottom-up manner.

\begin{example}[GCN Classifier and Computation Graph]
\label{exp:gcn}\rm
Given a simple example input graph shown in Figure~\ref{fig:computation_graph}(a), a 2-layer GCN  $f_{G}$ aims to predict the label of node $v_1$. 
Figure~\ref{fig:computation_graph}(b) shows the computation graph $\mathcal{T}_{v_1}$ of $v_1$ in the 2-layer GCN.
In the first layer of $f_{G}$, it computes the embeddings $\mathbf{h}_{v_1}^1$, $\mathbf{h}_{v_2}^1$, $\mathbf{h}_{v_3}^1$, and $\mathbf{h}_{v_4}^1$ of $v_1$, $v_2$, $v_3$, and $v_4$, respectively. By Equation~\ref{eq:vanilla_emebdding},  $\mathbf{h}_{v_{1}}^1=\rho(\frac{1}{4} \sum_{1 \leq i\leq 4}\mathbf{x}_{v_i} \mathbf{W}_1)$, $\mathbf{h}_{v_{2}}^1=\rho(\frac{1}{2} (\mathbf{x}_1 + \mathbf{x}_2) \mathbf{W}_1)$, $\mathbf{h}_{v_{3}}^1=\rho(\frac{1}{2} (\mathbf{x}_1 + \mathbf{x}_3) \mathbf{W}_1)$, and $\mathbf{h}_{v_{4}}^1=\rho(\frac{1}{3} (\mathbf{x}_1 + \mathbf{x}_2 + \mathbf{x}_4) \mathbf{W}_1)$. In the second layer of $f_{G}$, it computes the embedding $\mathbf{h}_{v_1}^2=\rho(\frac{1}{4}\sum_{1 \leq i\leq 4}\mathbf{h}_{v_i}^1 \mathbf{W}_2)$ of $v_1$. Finally, $f_{G}$ predicts the label of node $v_1$ as $\hat{y}_{v_1} = \sigma(\mathbf{h}^{2}_{v_1} \mathbf{W}_c)$.
\end{example}

Given a graph with an average node degree $d$, a $K$-layer GCN needs to aggregate information from $d^K$ neighbors for each node on average, which incurs high computation cost for large $d$~\cite{huangAdaptiveSamplingFast2018, liuBanditSamplersTraining2020}. Moreover, to compute the gradients of the loss function, for each $K$-hop neighbor $u$ of $v$, we need to store its all $K$ embeddings $\mathbf{h}_{u}^l\ (1 \leq l \leq K)$ in memory, which results in hefty memory footprints~\cite{DBLP:conf/iclr/LiuZYLCH22}.
To tackle the challenge, a sampling strategy can be used to train GCNs~\cite{DBLP:conf/nips/HamiltonYL17, DBLP:conf/kdd/YoonGSNHY21, congMinimalVarianceSampling2020} such that in each layer of a GCN, a node only aggregates embeddings from a small number of randomly sampled neighbors. Specifically, the sampling methods down-sample the computation graph $\mathcal{T}_v$ in a top-down manner such that each node in $\mathcal{T}_v$ contains at most $k$ child nodes. Then, the embedding $\mathbf{h}_v^{l}$ is computed by following the new computation graph. That is, 
$\mathbf{h}_{v}^{l} = \rho(\frac{1}{k+1} \sum_{v_i \in p(v) \cup \{v\}} \mathbf{h}_{v_i}^{l-1} \mathbf{W}_{l-1}),$ where $p(v)$ is a sample of $k$ child nodes of $v$ according to a sampling policy.
By regularizing the size of the computation graph in a GCN, the sampling methods can train the GCN in less cost. 
Figure~\ref{fig:computation_graph}(c) shows an example of down-sampled computation graph. Up to $1$ child node per node in $\mathcal{T}_{v_1}$ is sampled.

\begin{figure}[t]
    \centering
    \includegraphics[page=1, scale=0.15, trim=0 550 500 0, clip]{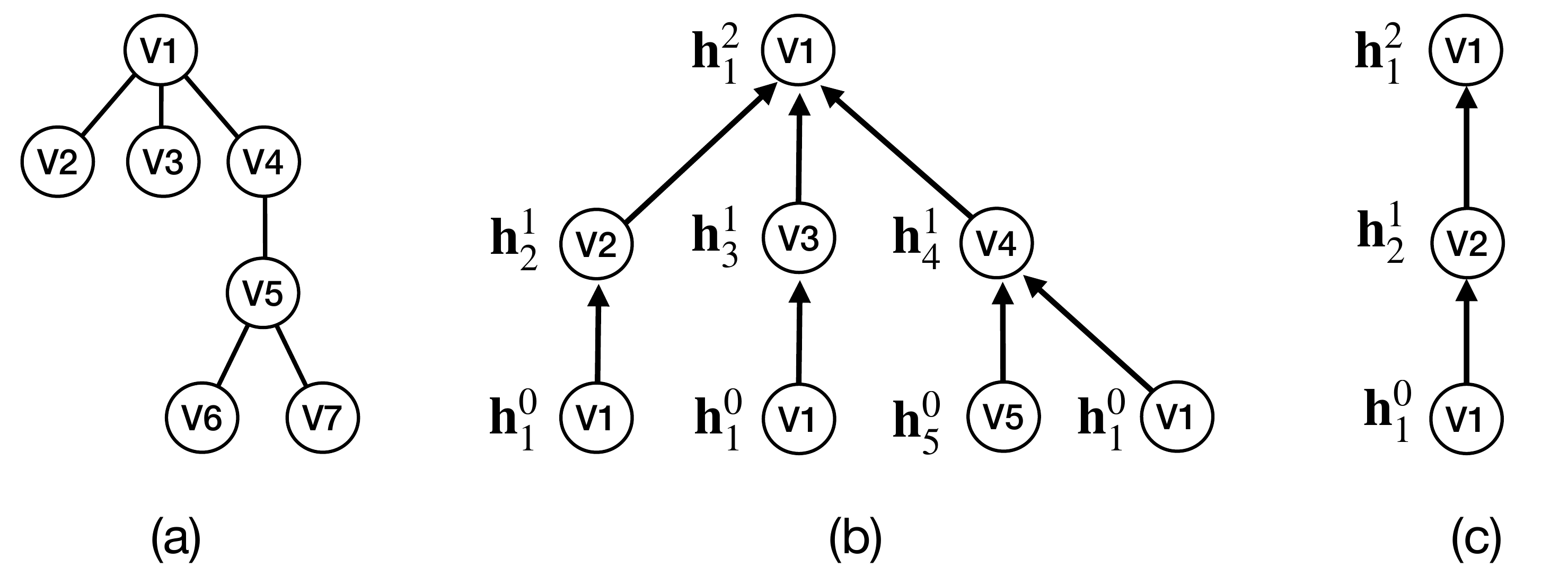}
    \caption{(a) An input graph  to a 2-layer GCN. (b) The original computation graph of node $v_1$ in the 2-layer GCN. (c) The down-sampled computation graph of node $v_1$ in the 2-layer GCN. The embedding of a node in a layer is plotted on the left of the node. The arrows indicate the directions of embedding aggregation. }
    \label{fig:computation_graph}
\end{figure}

\subsection{Fairness Notion of Demographic Parity}
\label{sec:preliminary_fairness}

Many notions of algorithmic fairness have been proposed in literature, such as demographic parity~\cite{DBLP:conf/icml/AgarwalBD0W18}, counterfactual fairness~\cite{DBLP:conf/wsdm/MaGWYZL22}, and individual fairness~\cite{DBLP:conf/innovations/DworkHPRZ12}. 
Due to its societal importance~\cite{DBLP:conf/ijcai/RahmanS0019, DBLP:journals/corr/abs-2201-00292}, demographic parity has been widely adopted in fair machine learning algorithms~\cite{DBLP:conf/wsdm/DaiW21, DBLP:conf/iclr/LiWZHL21, DBLP:conf/ijcai/ManishaG20, DBLP:conf/icml/AgarwalBD0W18}.

\begin{definition}[Demographic Parity \cite{DBLP:conf/icml/AgarwalBD0W18}] Given  random samples $(\mathbf{x}, s)$ drawn from a joint distribution over a feature space $\mathcal{X}$ and a sensitive attribute $\mathcal{A}$, a classifier $f(\cdot): \mathbb{R}^{d} \rightarrow \{0, 1\}$ is said to satisfy \textbf{demographic parity} if the prediction $f(\mathbf{x})$ is independent of the sensitive attribute value $s$, that is, $\forall a \in \mathcal{A}$, $\mathbb{E}[f(\mathbf{x})|s=a] = \mathbb{E}[f(\mathbf{x})]$.
\label{eq:dp}
\end{definition}

To train a fair classifier, in-processing fairness methods add one or more demographic parity penalty terms $\mathcal{L}_{\text{dp}}$ in the loss function of the classifier to penalize discriminatory behaviors. 
In particular, the traditional regularization methods~\cite{ DBLP:conf/wsdm/DaiW21, DBLP:conf/icml/AgarwalBD0W18, DBLP:conf/ijcai/ManishaG20} train a fair GCN classifier by optimizing the \emph{loss of fairness and utility}, defined by
\begin{equation}
    \mathcal{L} = \mathcal{L}_{\text{accuracy}} + \alpha \mathcal{L}_{\text{dp}},
    \label{eq:fair_loss}
\end{equation}
where  $\alpha$ is the weight of the fairness penalty.  

There is a well recognized tradeoff between model accuracy and demographic parity~\cite{DBLP:journals/corr/abs-2201-00292, DBLP:conf/uai/AgarwalLZ21, DBLP:conf/fat/MenonW18, DBLP:journals/jmlr/ZhaoG22}. GCNs use feature aggregation mechanisms over graph topology to learn node representations. Due to this unique property, applying Equation~\ref{eq:fair_loss} to regularize the model parameters alone may not be sufficient to train GCNs with good accuracy and fairness.
In this paper, we study how to rectify the feature aggregation process of GCNs by modifying their computation graphs so that fairness can be improved with minor tradeoff in utility.

\section{Intuition and Problem Formulation}
\label{sec:problem_formulation}
\label{sec:analyze}

Since GCNs perform node classification by aggregating and transforming node feature vectors over graph structures~\cite{DBLP:conf/kdd/JiaLYYLA20}, biased graph connections may lead to poor demographic parity~\cite{DBLP:conf/wsdm/DaiW21}. To fix the issue and strengthen demographic parity in node classification, one idea is to modify connections in input graphs~\cite{DBLP:conf/iclr/LiWZHL21, DBLP:conf/aistats/LaclauRCL21}.

To elaborate the insight of modifying input graph structures to improve demographic parity, for the sake of simplicity, let us analyze the fairness of \emph{simple GCNs}~\cite{DBLP:conf/icml/WuSZFYW19}, which are a type of GCNs without the non-linear activation function in Equation~\ref{eq:vanilla_emebdding}. 
Simple GCNs are widely adopted in recommender systems, thanks to its good scalability and comparable or even improving performance to general GCNs with nonlinear activation functions~\cite{DBLP:conf/sigir/0001DWLZ020, DBLP:conf/aaai/ChenWHZW20, DBLP:conf/sigir/WangJZ0XC20, DBLP:journals/corr/abs-2011-02260}.
We leave the theoretical analysis on general GCNs to future work.

Denote by $\tilde{\mathbf{A}} \in \mathbb{R}^{n \times n}$ a \emph{random walk probability matrix}, where each element $\tilde{\mathbf{A}}{[i, j]}$ $(1\leq i \leq n, 1\leq j \leq n)$ is the probability of a $K$-step random walk starting at $v_i$ and terminating at $v_j$. Let $\beta{[\neq a, a]} = \frac{1}{|\mathcal{V}_{\neq a}|} \sum_{v_i \in \mathcal{V}_{\neq a}, v_j \in \mathcal{V}_{a}} \tilde{\mathbf{A}}{[i, j]}$ be the average probability for a $K$-step random walk starting at a node with sensitive value not equal to $a$ and terminating at a node with sensitive value $a$. Similarly, denote by $\beta{[a, a]}$ the average probability for a $K$-step random walk starting at a node with sensitive value $a$ and terminating at a node with the same sensitive value.

Following the idea in logit matching~\cite{DBLP:conf/ijcai/ShinYC19}, an empirical counterpart of the regularization term $\mathcal{L}_{\text{dp}}$ in Equation~\ref{eq:fair_loss} is $\hat{\mathcal{L}}_{\text{dp}} = \sum_{a \in \mathcal{A}} \left\lVert \frac{\sum_{v_i \in \mathcal{V}_a} \mathbf{p}_{v_i} }{|\mathcal{V}_a|} - \frac{\sum_{v_j \in \mathcal{V}_{\neq a}} \mathbf{p}_{v_j} }{|\mathcal{V}_{\neq a}|} \right\rVert_2,$
where a logit $\mathbf{p}_v = \mathbf{h}_{v}^{K}\mathbf{W}$ is the input vector to the softmax function of a graph neural network. $\hat{\mathcal{L}}_{\text{dp}}$ measures the disparity between the prediction logits of sensitive groups.

Next, we show that we can reduce the upper bound of the demographic parity regularizer $\hat{\mathcal{L}}_{\text{dp}}$ on GCNs by modifying the connections in the input graph. Thus, we may train GCNs with better demographic parity on the modified graph. 

\begin{theorem}
Denote by $\mathbf{\mu}_{a} = \frac{1}{|\mathcal{V}_{a}|}\sum_{v \in \mathcal{V}_{a}}\mathbf{x}_v$ the mean of the node feature vectors in a group $\mathcal{V}_{a}$ and by $dev(\mathcal{V}_a) = \max_{v \in \mathcal{V}_a}\{\| v - \mu_{a}\|_{\infty}\} $ the deviation. Let $\delta_a=\max\{dev(\mathcal{V}_a), dev(\mathcal{V}_{\neq a})\}$. For any SGCN with parameters $\mathbf{W}$, an upper bound of the regularization term $\hat{\mathcal{L}}_{\text{dp}}$ is
\footnotesize{
\begin{equation}
    \hat{\mathcal{L}}_{\text{dp}} 
    \leq
\sum_{a \in \mathcal{A}} \lVert \mathbf{W} \rVert_2  (|\beta{[a, a]} - \beta{[\neq a, a]}| \cdot 
\lVert \mu_{a} - \mu_{\neq a} \rVert_2 + 2 \sqrt{d} \sum_{a \in \mathcal{A}} \delta_{a}).
\label{eq:dp_upper_bound}
\end{equation}
}
\label{proposition}
\end{theorem}

\begin{proof}
$\hat{\mathcal{L}}_{\text{dp}} 
    = 
    \sum_{a \in \mathcal{A}} \left\lVert \frac{\sum_{v_i \in \mathcal{V}_a} \mathbf{p}_{v_i} }{|\mathcal{V}_a|} - \frac{\sum_{v_j \in \mathcal{V}_{\neq a}} \mathbf{p}_{v_j} }{|\mathcal{V}_{\neq a}|} \right\rVert_2$ by definition.
Wu \textit{et al.}~\cite{DBLP:conf/icml/WuSZFYW19} show that the prediction of a $K$-layer SGCN on a node $v$ is 
$\hat{y}_{v} = \sigma( \mathbf{h}_v \mathbf{W})$, where $\mathbf{h}_v =  \sum_{u \in \mathcal{V}} \tilde{\mathbf{A}}[v, u] \mathbf{x}_u$.
Plugging in the definition of $p_i = \mathbf{h}_{v_i} \mathbf{W}$,  the term $\hat{\mathcal{L}}_{\text{dp}}^a = \lVert \frac{\sum_{v_i \in \mathcal{V}_a} p_i}{|\mathcal{V}_a|} - \frac{\sum_{v_j \in \mathcal{V}_{\neq a}} p_j}{|\mathcal{V}_{\neq a}|} \rVert_2$ in $\hat{\mathcal{L}}_{\text{dp}}$ can be rewritten as
$\hat{\mathcal{L}}_{\text{dp}}^a 
=
\lVert 
\frac{\sum_{v_i \in \mathcal{V}_a} \mathbf{h}_{v_i} \mathbf{W}}{|\mathcal{V}_a|} - \frac{\sum_{v_j \in \mathcal{V}_{\neq a}} \mathbf{h}_{v_j} \mathbf{W}}{|\mathcal{V}_{\neq a}|} 
\rVert_2$. 
As the Lipschitz constant for a linear function $g(\mathbf{x}) = \mathbf{x} \mathbf{W}$ is the spectral normal of $\mathbf{W}$~\cite{DBLP:conf/iclr/MiyatoKKY18}, we have 
\begin{equation}
    \hat{\mathcal{L}}_{\text{dp}}^a 
\leq 
\lVert \mathbf{W} \rVert_2
\left\lVert 
\frac{\sum_{v_i \in \mathcal{V}_a} \mathbf{h}_{v_i}}{|\mathcal{V}_a|} - \frac{\sum_{v_j \in \mathcal{V}_{\neq a}} \mathbf{h}_{v_j}}{|\mathcal{V}_{\neq a}|} 
\right\rVert_2.
\label{eq:ldp_temp_upper_bound}
\end{equation}

For simplicity, we use the bracket notation to represent the range of a vector\footnote{Consider three vectors $\mathbf{x}_1$, $\mathbf{x}_2$, and  $\mathbf{x}_3$ with the same dimensionality $d$. If $\mathbf{x}_1$ satisfies that, for $i \in [1, d]$, $ \mathbf{x}_2[i] - \mathbf{x}_3[i] \leq \mathbf{x}_1[i] \leq  \mathbf{x}_2[i] + \mathbf{x}_3[i]$, we write $\mathbf{x}_1 \in [\mathbf{x}_2 \pm \mathbf{x}_3]$. }. 
By the definition of $\delta_a$, for a node $v \in \mathcal{V}_a$, we have $\mathbf{x}_{v} \in [\mathbf{\mu}_{a} \pm \delta_a * \mathbbm{1}]$, where $\mathbbm{1} \in \mathbb{R}^{d}$ is an all one vector. 
By the definition of $\mathbf{h}_{v}$, we have
$\mathbf{h}_{v_{i}} \in 
[ \sum_{v_j \in \mathcal{V}_{a} } \tilde{\mathbf{A}}[i, j] \mathbf{\mu}_a + \sum_{v_j \in \mathcal{V}_{\neq a} } \tilde{\mathbf{A}}[i, j]  \mathbf{\mu}_{\neq a} \pm \delta_{a} * \mathbbm{1}]$. Thus, the mean node embedding
$\frac{\sum_{v_i \in \mathcal{V}_{a}} \mathbf{h}_{v_i}}{|\mathcal{V}_a|} 
\in
[
\beta{[a, a]} \mu_{a} + \beta{[a, \neq a]} \mu_{\neq a} \pm \delta_{a} * \mathbbm{1}
]$. 
Similarly, we have the mean node embedding
$\frac{\sum_{v_i \in \mathcal{V}_{\neq a}} \mathbf{h}_{v_i}}{|\mathcal{V}_{\neq a}|} 
\in
[
\beta[\neq a, a] \mu_{a} + \beta[\neq a, \neq a] \mu_{\neq a} \pm \delta_a * \mathbbm{1}
]
$. Therefore, we have $
\frac{\sum_{v_i \in \mathcal{V}_{a}} \mathbf{h}_i^K}{|\mathcal{V}_a|}  - \frac{\sum_{v_i \in \mathcal{V}_{\neq a}} \mathbf{h}_i^K}{|\mathcal{V}_{\neq a}|}  
\in 
[
(\beta[a, a] - \beta[\neq a, a]) \mu_{a} + (\beta[a, \neq a] - \beta[\neq a, \neq a]) \mu_{\neq a} \pm 2\delta_{a} * \mathbbm{1}
]
$.
Since $\beta[a, \neq a] = 1 - \beta[a, a]$ and $\beta[\neq a, \neq a] = 1 - \beta[\neq a, a]$, the vector range can be written as 
$
\frac{\sum_{v_i \in \mathcal{V}_{a}} \mathbf{h}_i^K}{|\mathcal{V}_a|}  - \frac{\sum_{v_i \in \mathcal{V}_{\neq a}} \mathbf{h}_i^K}{|\mathcal{V}_{\neq a}|}  
\in 
[
(\beta[a, a] - \beta[\neq a, a])(\mu_{a} - \mu_{\neq a}) \pm 2\delta_a * \mathbbm{1}
]
$.
Thus, 
$\lVert \frac{\sum_{v_i \in \mathcal{V}_{a}} \mathbf{h}_i^K}{|\mathcal{V}_a|}  - \frac{\sum_{v_i \in \mathcal{V}_{\neq a}} \mathbf{h}_i^K}{|\mathcal{V}_{\neq a}|} \rVert_2
\leq
|\beta[a, a] - \beta[\neq a, a]| * 
\lVert \mu_{a} - \mu_{\neq a} \rVert_2 + 2\sqrt{d} \delta_a
$.
Plugging the above inequality into Equation~\ref{eq:ldp_temp_upper_bound}, we have $\hat{\mathcal{L}}_{\text{dp}}^{a} 
\leq 
\lVert \mathbf{W} \rVert_2 * (|\beta[a, a] - \beta[\neq a, a]| * 
\lVert \mu_{a} - \mu_{\neq a} \rVert_2 + 2\sqrt{d} \delta_a)$. 

Plugging in the upper bound of each term in $\hat{\mathcal{L}}_{\text{dp}}$, we have Equation~\ref{eq:dp_upper_bound}.
\end{proof}

Equation~\ref{eq:dp_upper_bound} indicates three sources of poor demographic parity of GCNs, namely, node attribute bias $\lVert \mu_{a} - \mu_{\neq a} \rVert_2$, graph structure bias $|\beta[a, a] - \beta[\neq a, a]|$, and model parameters $\mathbf{W}$. 
In the literature, some techniques have been developed to reduce node attribute bias and model parameter bias~\cite{DBLP:journals/kais/KamiranC11, DBLP:conf/nips/CalmonWVRV17, DBLP:conf/ijcai/ManishaG20, DBLP:conf/wsdm/DaiW21}.
In addition to debiasing node features and GCN parameters, Equation~\ref{eq:dp_upper_bound} suggests that we can further improve the demographic parity of a GCN by mitigating the graph structure bias.
In particular, we can regulate the graph structure such that $\beta{[a, a]} = \beta{[\neq a, a]}$.

A solution that eliminates graph structure biases is to process the graph structures by adding and removing edges so that each node $v$ has a \textit{balanced neighborhood}, that is, $\Gamma_v$ has the same number of nodes from each sensitive group. Formally, $\forall a_{i}, a_{j} \in \mathcal{A}$, we want $|\Gamma_{v}  \cap \mathcal{V}_{a_{i}}| = |\Gamma_{v} \cap \mathcal{V}_{a_{j}}|$. 

\begin{figure*}[t]
\centering
\subfloat[]{
\includegraphics[page=2, width=0.16\linewidth, scale=0.3, trim=0 650 1600 0, clip]{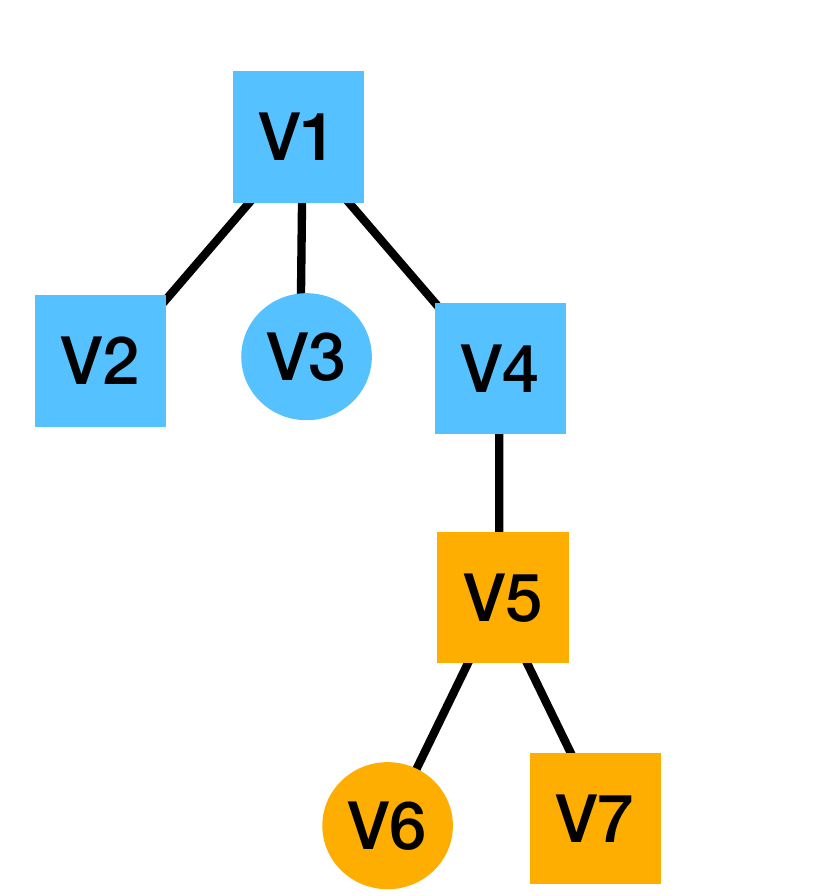}
\label{figure:origin_graph}
}
\subfloat[]{
\includegraphics[page=3, width=0.194\linewidth, scale=0.3, trim=0 500 1450 0, clip]{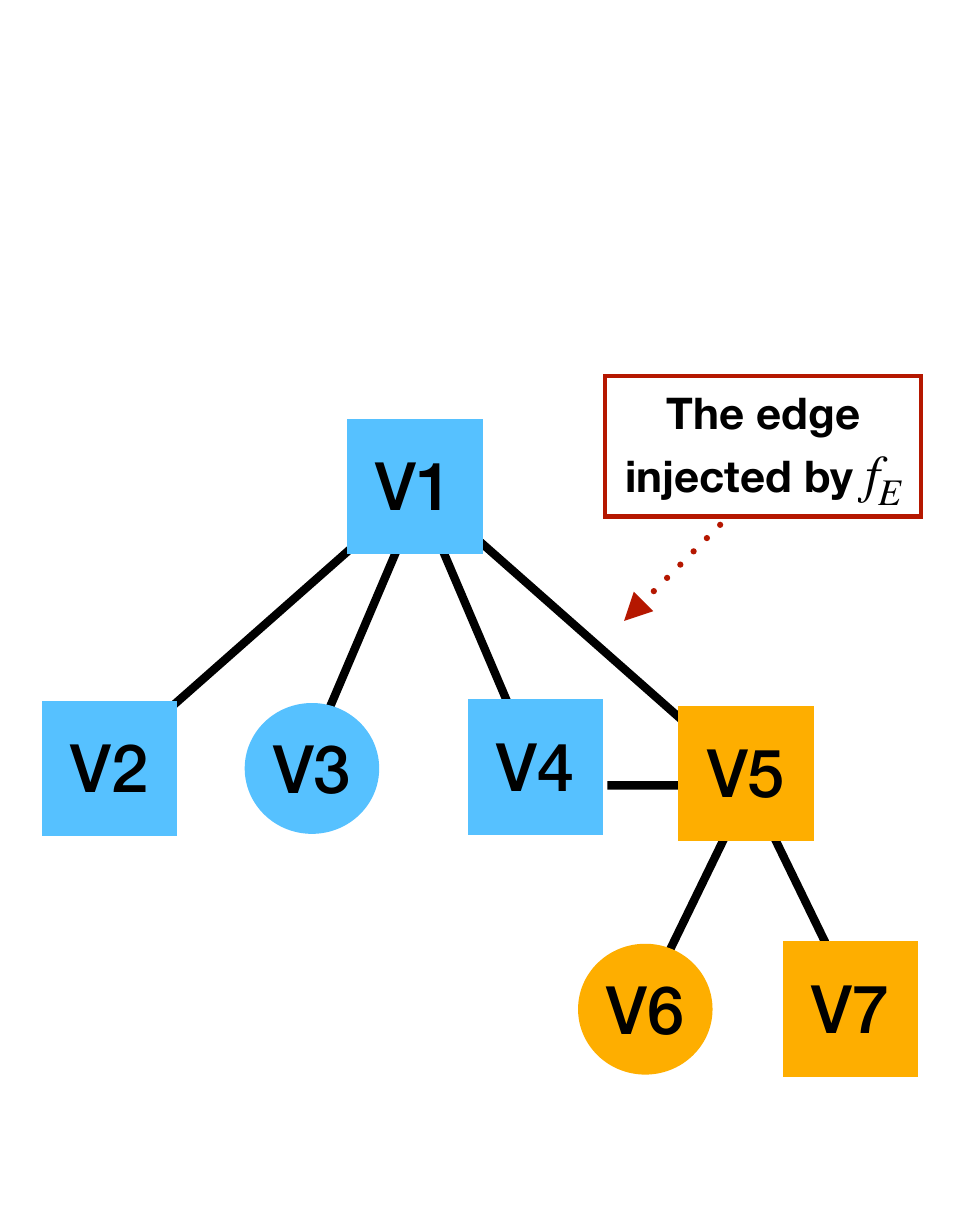}
\label{figure:add_edges}
}
\subfloat[]{
\includegraphics[page=4, width=0.589\linewidth, scale=0.3, trim=0 400 500 0, clip]{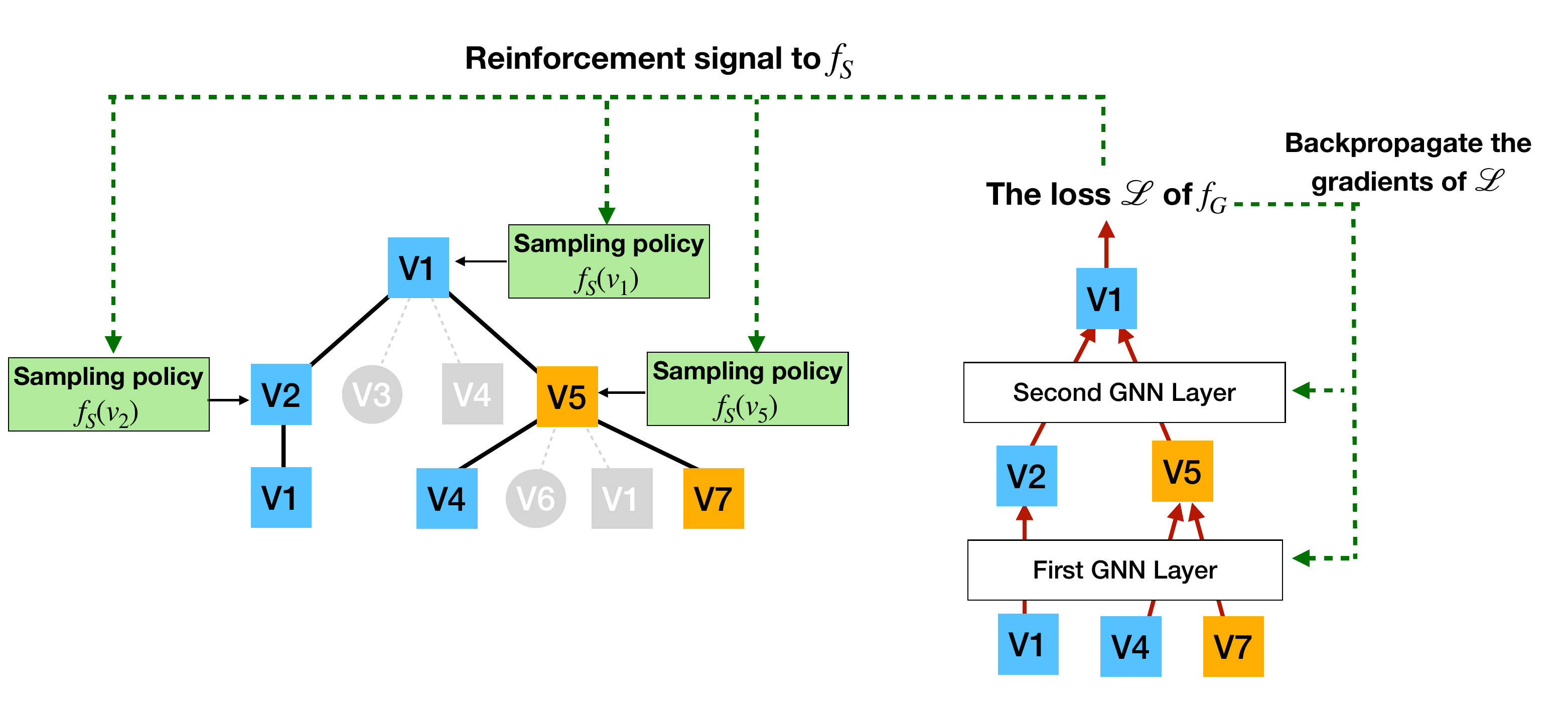}
\label{figure:sample_edges}
}
\caption{An example showing the intuition of the FairSample approach. (a) An input graph to FairSample.  (b) The augmented graph after injecting an inter-group edge. (c) Jointly train the sampling policy $f_{S}$ and the 2-layer GCN node classifier $f_{G}$ with the computation graph of a node $v_1$. }
\label{fig:framework}
\end{figure*}

A remaining question is which edges should be changed.  A straightforward strategy is to pre-process the input graph $\mathcal{G}$ by randomly adding and removing edges so that each node has a balanced neighborhood. This method, however, may not work well, as $\mathcal{G}$ is modified without considering the utility for training GCNs.
As a result, too many noisy edges may be introduced into $\mathcal{G}$ and the GCNs trained on the changed graph may not be useful~\cite{DBLP:conf/kdd/Dai0W21, DBLP:conf/aaai/0003LNW0S21}. Thus, an ideal method for graph structure bias reduction should alter the edges in a way that minimizes $|\beta[a,a] - \beta[\neq a, a]|$ and maintains the prediction power of GCNs built on the modified graph. This is our major insight driving the development of our method in this paper.

Based on the above intuition, we define the problem of \emph{learning a demographic parity fairness-aware sampling strategy for GCNs}.
Given a graph $\mathcal{G} = (\mathcal{V}, \mathcal{E}, \mathbf{X})$, where there is a sensitive attribute $\mathcal{A} \in \mathcal{X}$ and the sensitive values of the nodes are $\mathbf{S} = \{s_1, \ldots, s_n\}$, and a set of labeled node $\mathcal{V}_{L} \subseteq \mathcal{V}$ with corresponding binary labels $Y_L$, we aim to design a sampling strategy that can mitigate graph structure bias and learn a fair graph convolutional neural network
$
    f(\mathcal{G}, \mathbf{S}, \mathcal{V}_{L}, Y_L) \rightarrow \hat{Y}_U,
$
where $\hat{Y}_U$ is the set of predicted labels of unlabeled nodes $U=\mathcal{V} \setminus \mathcal{V}_L$. 
The learned GCN should have high accuracy and good demographic parity.

\section{FairSample: Sampling for Fair GCNs}
\label{sec:proposed_method}
In this section, motivated by the analysis in Section~\ref{sec:problem_formulation}, we develop FairSample, a two-phase framework for training demographic parity fair GCNs and retaining scalability and accuracy.

\subsection{The Framework of FairSample}
\label{sec:framework}

A graph is \emph{homophilic} with respect to the sensitive values if the nodes with the same sensitive value are more likely to be connected.
In this paper, we assume homophilic input graphs, which are widely observed in real world social networks~\cite{DBLP:conf/wsdm/DaiW21, DBLP:conf/aistats/LaclauRCL21, DBLP:conf/ijcai/RahmanS0019}. 
Our method can be easily adapted to heterophilic graphs as discussed at the end of this subsection, where the nodes with the same sensitive value are less likely to be connected.

Our analysis in Section~\ref{sec:analyze} suggests that we can strengthen demographic parity in node classification by learning node embeddings with balanced neighborhoods.  Dai~\textit{et~al.}~\cite{DBLP:conf/kdd/Dai0W21} suggest that the accuracy in node classification can be improved by aggregating embeddings from informative neighbors. A neighbor node $u$ of a node $v$ is \emph{informative} with respect to $v$, if $u$ and $v$ have similar node feature vectors~\cite{DBLP:conf/kdd/Dai0W21, DBLP:conf/kdd/YoonGSNHY21}. Based on the two intuitions, our framework trains a fair and accurate GCN by modifying the computation graphs in training a GCN so that each node aggregates embeddings from a set of balanced and informative neighbors.  

There are two major challenges. First, we have to handle nodes with dominating intra-group connections, which may lead to bias in graph structure. \textit{Intra-group connections} are edges between the nodes with the same sensitive values. \textit{Inter-group connections} are edges between the nodes with different sensitive values. If all neighbors of a node $v$ have the same sensitive values, there is no way for $v$ to aggregate embeddings from diverse sensitive groups. Second, we need to sample the edges for embedding aggregation in GCN training so that we can learn GCNs with good accuracy, demographic parity, and scalability. An edge sampling method not carefully designed may not help with either model accuracy~\cite{DBLP:conf/kdd/Dai0W21, DBLP:conf/aaai/0003LNW0S21} or model demographic parity~\cite{DBLP:conf/iclr/LiWZHL21, DBLP:conf/aistats/LaclauRCL21}.

FairSample tackles the two challenges in a two-phase framework. In phase 1, FairSample uses an edge injector $f_{E}$ to augment the input graph $\mathcal{G}$ by adding inter-group edges to improve the inter-group connectivity of the graph. As the injected edges may introduce noise into the data, the accuracy of the GCN classifier $f_{G}$ built on the modified graph may be hurt~\cite{DBLP:conf/kdd/Dai0W21}. To alleviate the problem, we only connect similar nodes with the same class labels.  As justified by Dai~\textit{et~al.}~\cite{DBLP:conf/kdd/Dai0W21}, adding such connections can benefit the accuracy of GCNs. Phase 1 improves the heterophilic neighborhood for nodes with dominated intra-group connections. 

\begin{example}[Phase 1]\rm
Consider the input graph in Figure~\ref{figure:origin_graph}, where each node has a binary sensitive value (represented by the colors of the nodes) and a binary class label (represented by the shapes of the nodes). Node $v_5$ and node $v_1$ are similar nodes with the same class label, since the distance between the two nodes in the graph is short, only $2$.  At the same time, they have different sensitive values. Thus, in Phase 1, FairSample may connect the two nodes and obtain the graph shown in Figure~\ref{figure:add_edges}.
\end{example}

In phase 2, FairSample uses reinforcement learning to learn a computation graph sampling policy $f_{S}$. The policy samples the edges in the computation graphs for training $f_{G}$ such that the loss of fairness and utility (Equation~\ref{eq:fair_loss}) of $f_{G}$ can be minimized. 
FairSamlpe jointly learns $f_{G}$ and $f_{S}$ on the graph augmented by $f_{E}$. 

\begin{example}[Phase 2]\rm
Figure~\ref{figure:sample_edges} shows the training process of $f_{S}$ and a 2-layer GCN $f_{G}$ using the augmented input graph in Figure~\ref{figure:add_edges}, where the left part shows the sampling operations on the computation graph $\mathcal{T}$ of $v_1$ and the right part shows the computation of the loss function $\mathcal{L}$ of $f_{G}$ following the sampled computation graph.
The policy $f_{S}$ samples $\mathcal{T}$ in a top-down breadth-first manner. Up to 2 child nodes per node in $\mathcal{T}$ are sampled.
First, the child nodes $v_2$ and $v_5$ of the root node $v_1$ are sampled.
Then, $f_{S}$ iteratively samples the child nodes of $v_2$ and $v_5$ to obtain the sampled computation graph $\mathcal{T}'$. The classifier $f_{G}$ computes the loss function $\mathcal{L}$ following $\mathcal{T}'$ and optimizes its parameters using the backpropagation. FairSample uses the value of $\mathcal{L}$ as the reinforcement signal to train the policy $f_{S}$ to minimize $\mathcal{L}$.

Through this reinforcement learning process, $f_{S}$ helps each node in $\mathcal{T}'$ to aggregate embeddings from a set of neighbors that can improve the demographic parity and accuracy of $f_{G}$.
\label{example:phase2}

\end{example}


To adapt FairSample to heterophilic input graphs, we can use the edge injector $f_{E}$ to add intra-group edges and enhance the graph's intra-group connectivity.

\subsection{Edge Injector}
\label{sec:edge_injector}

Theorem~\ref{proposition} suggests that injecting inter-group edges into the input graph may improve the demographic parity of the GCN built on the modified graph. 
To improve the demographic parity and retain the prediction quality of the GCN being trained, we pre-process the input graph by adding inter-group edges connecting similar nodes with the same class labels.  A challenge in implementing the strategy is that we only know the labels of $\mathcal{V}_{L}$, which usually only contains a very small number of nodes. As a result, we can only inject edges between a small number of nodes and the benefits from our edge injection strategy may be severely limited. To tackle the challenge of label sparsity, we can add accurate pseudo labels to some of the unlabeled nodes~\cite{DBLP:conf/kdd/Dai0W21, DBLP:conf/cvpr/IscenTAC19}. Let $\mathcal{V}'_{L} = \mathcal{V}_L \cup \mathcal{V}_P$ be the augmented set of labeled nodes, where $\mathcal{V}_P$ is the set of nodes with pseudo labels. With more labeled and pseudo-labeled nodes, our edge injection strategy can create more inter-group edges to better mitigate graph structure bias.

We use a GCN as the pseudo label predictor, which utilizes both node features and graph structures to make accurate label predictions. To obtain accurate pseudo labels, we follow the existing work~\cite{DBLP:conf/kdd/Dai0W21, DBLP:conf/cvpr/IscenTAC19} and only keep the pseudo labels with confidence scores larger than a threshold $\tau$. Our framework can take any GNN models as the pseudo label predictor.

Next, we introduce our edge injection strategy. 
As two nodes with a small shortest-path distance in the graph are more similar~\cite{DBLP:conf/kdd/YenSMS08}, we limit that two nodes $v_i, v_j \in \mathcal{V}'_L$ are connected by $f_{E}$ only if their shortest-path distance is smaller than or equal to $h$.
Specifically, our edge injector proceeds as follows. 

Let $C_v$ be the set of nodes in the $h$-hop neighborhood of $v$ that have the same class labels as $v$ but different sensitive values from $v$. 
For each node $v \in \mathcal{V}'_L$,
$f_{E}$ uniformly randomly picks at most $\min\{m, |C_v|\}$ nodes in $C_v$ and connects them with $v$. We set $h$ to 2 in this work to reduce the computational cost for searching $C_v$ and avoid the over-smooth issue~\cite{DBLP:conf/icml/ChenWHDL20, DBLP:conf/kdd/LiuGJ20} that may significantly degrade the accuracy of GCNs. 
As $f_{E}$ injects inter-group edges in a heuristic manner, it may introduce noisy edges into the data and negatively affect the utility of the GCNs being trained.
However, as we apply the computation graph sampler $f_{S}$ to pick informative neighbors for GCN training, the noisy edges may not be used and thus their negative effect on the accuracy of the GCN may be reduced. 
This is because the sampler is trained to optimize the loss of demographic parity and utility of the target GCN classifier. If an edge significantly degrades the utility of the target GCN classifier, that edge may not be selected by the sampler.
Our empirical study in Section~\ref{exp:ablation} shows that the edge injection strategy can improve the fairness of the trained GCN classifiers without sacrificing accuracy in most cases.

\subsection{Computation Graph Sampler}
\label{sec:computation_graph_sampler}
Now, we introduce our computation graph sampler $f_{S}$ for training a fair, accurate, and scalable GCN $f_{G}$. 
The sampler $f_{S}$ is learned by reinforcement learning to optimize the loss $\mathcal{L}$ of fairness and utility of $f_{G}$.

An ideal computation graph sampling method should improve the demographic parity of the target GCN classifier $f_{G}$ and maintain the utility of $f_{G}$. Stratified node sampling  with respect to the sensitive values can select a balanced set of nodes for feature aggregation, which benefits the demographic parity of $f_{G}$ according to Theorem~\ref{proposition}.
This method,
however, may not train an accurate $f_{G}$, as it does not consider the informativeness of the nodes. Node similarity sampling for feature aggregation, on the one hand, can benefit the accuracy of $f_{G}$~\cite{DBLP:conf/aaai/0003LNW0S21, DBLP:conf/kdd/YoonGSNHY21} and, on the other hand, cannot improve the demographic parity of $f_{G}$, as it cannot reduce the structure bias of the input graph. 

To train GCNs with good demographic parity and high accuracy, we propose a computation graph sampler $f_{S}$, which leverages reinforcement learning to combine the sampling results from the stratified sampling method and the similarity sampling method. Our key idea in $f_{S}$ is to sample the computation graphs $\mathcal{T}$ in the GCN $f_{G}$ such that each node in the sampled computation graph aggregates embeddings from a balanced and informative set of child nodes. Specifically, $f_{S}$ down-samples $\mathcal{T}$ in a top-down breadth-first manner. The sampler $f_{S}$ starts from the root node $v$ of $\mathcal{T}$. Following the sampling framework by Yoon~\textit{et~al.}~\cite{DBLP:conf/kdd/YoonGSNHY21}, $f_{S}$ samples $k$ child nodes of $v$ with replacement and discards the duplicate sampled nodes to get a set of up to $k$ sampled child nodes $p(v)$.  A child node $v_i$ of $v$ is sampled following the probability $\text{P}(v_i| v)$, which is a parameterized function taking the feature vectors of $v_i$ and $v$ as input. 
Then, for each $v_i \in p(v)$, $f_{S}$ samples a set of child nodes $p(v_i)$ by repeating the above steps. The sampler $f_{S}$ iteratively operates on each sampled node in $\mathcal{T}$ until reaching the leaves of $\mathcal{T}$.  We use sampling with replacement here to ensure independence among samples.

After getting the sampled computation graphs $\mathcal{T}'$, we apply the classifier $f_{G}$ to compute node predictions following $\mathcal{T}'$ and calculate the loss function $\mathcal{L}$ correspondingly. Then, we update the parameters of $f_{S}$ and $f_{G}$ to minimize $\mathcal{L}$. We update the parameters of $f_{G}$ using backpropagation.
Since the sampling operation is non-differentiable~\cite{huangAdaptiveSamplingFast2018}, we cannot directly train the parameters of $f_{S}$ using the gradients of $\mathcal{L}$. 
To address this challenge, FairSample uses the log derivative technique in reinforcement learning~\cite{DBLP:journals/ftml/Francois-LavetH18} to train sampling policies.

The probability $\text{P}(v_j | v_i)$ of sampling a child node $v_j$ of  $v_i$ is computed based on two components, node similarity sampling $q_{\text{sim}}(v_j | v_i)$ and stratified sampling $q_{\text{fair}}(v_j | v_i)$. 
The component $q_{\text{sim}}(v_j | v_i) = (\mathbf{x}_i \mathbf{W}_{s}) \cdot (\mathbf{x}_j \mathbf{W}_{s})$ is designed to improve model accuracy, where $\mathbf{W}_{s} \in \mathbb{R}^{d \times d_s}$ is a feature transformation matrix and $d_s$ is the dimensionality of the transformed features. It assigns higher probabilities to the child nodes that have similar features as $v_i$, and thus helps to sample informative child nodes that can benefit the accuracy of GCNs~\cite{DBLP:conf/aaai/0003LNW0S21}.
The component $q_{\text{fair}}(v_j | v_i) = \frac{1}{|\Gamma_{v_i} \cap \mathcal{V}_{s_j}|}$ is designed to improve model fairness. It assigns higher probabilities to the minority sensitive groups in the child nodes $\Gamma_{v_i}$ of $v_i$, and thus helps to sample a set of child nodes with balanced sensitive values.
An attention mechanism is applied to perform a tradeoff between the two components. Specifically, we define $\text{P}(v_j | v_i) = \frac{q(v_j | v_i)}{\sum_{v \in \Gamma_{v_i}} q(v | v_i)}$, where $q(v_j | v_i) = \mathbf{a} \cdot \langle q_{\text{sim}}(v_j | v_i), q_{\text{fair}}(v_j | v_i) \rangle$ and $\mathbf{a} \in \mathbb{R}^2$ is an attention vector. 
The attention vector $\mathbf{a}$ learns which component is more effective on optimizing the loss $\mathcal{L}$ of $f_{G}$.
Both $\mathbf{W}_{s}$ and $\mathbf{a}$ are learnable parameters of our sampling policy and are shared by all nodes in $\mathcal{V}$.

FairSample borrows the training mechanism by Yoon~\textit{et~al.}~\cite{DBLP:conf/kdd/YoonGSNHY21} to learn the parameters $\theta = (\mathbf{W}_{s}, \mathbf{a})$ of $f_{S}$.
The gradient of $\mathcal{L}$ with respect to $\theta$ is estimated as $\nabla_{\theta} \mathcal{L} \approx \frac{d}{d \theta}( \frac{1}{|\mathcal{T}_1'|} \sum_{v_i \in \mathcal{T}'_1} \frac{d \mathcal{L}}{d \mathbf{h}^{1}_{v_i}} \mathbb{E}_{v_j \in p(v_i)}[\log \text{P}(v_j | v_i) \mathbf{x}_j]),$
where $\mathcal{T}'_1$ is the set of nodes in the first layer of the sampled computation graph $\mathcal{T}'$, $\mathbf{h}^{1}_{v_i}$ is the embedding of $v_i$ following $\mathcal{T}'$, and $p(v_i)$ is the set of child nodes of $v_i$ in $\mathcal{T}'$.

\begin{example}\rm
Continue the training process of $f_{S}$ in Figure~\ref{figure:sample_edges}, which is discussed in Example~\ref{example:phase2}.  Among the child nodes of the root node $v_1$, $v_5$ has the largest sampling probability, as $v_5$ not only has features similar to those of $v_1$ (large $q_{sim}(v_5|v_1)$) but also belongs to the minority sensitive group in $\Gamma_{v_1}$ (large $q_{\text{fair}}(v_5 | v_1)$). 
\label{exp:sampler}
\end{example}

The optimization framework~\cite{DBLP:conf/kdd/YoonGSNHY21} utilized by FairSample is principled and general. It can readily accommodate various definitions of $\mathcal{L}_{dp}$ as long as $\mathcal{L}_{dp}$ is differentiable, such as the many proposed in literature~\cite{DBLP:conf/wsdm/DaiW21, DBLP:conf/iclr/LiWZHL21, DBLP:conf/ijcai/ManishaG20, DBLP:conf/icml/BoseH19}. In literature, some definitions of $\mathcal{L}_{dp}$ are based on the adversarial learning framework~\cite{DBLP:conf/wsdm/DaiW21, DBLP:conf/icml/BoseH19} and are prone to be unstable during training~\cite{DBLP:conf/aies/BeutelCDQWLKBC19}. 
In our experiments, we choose the well-adopted definition~\cite{DBLP:conf/ijcai/ManishaG20}, that is, $$\mathcal{L}_{dp}=\sum_{a \in \mathcal{A}} \left| \frac{\sum_{v_i \in \mathcal{V}_a} f(v_i) }{|\mathcal{V}_a|} - \frac{\sum_{v_j \in \mathcal{V}_{\neq a}} f(v_j) }{|\mathcal{V}_{\neq a}|} \right|,$$ where $f(v_i)$ is the predicted probability of the positive label, and $\mathcal{V}_{a_i}$ and $\mathcal{V}_{\neq a_i}$ are the sets of nodes taking and not taking the sensitive value $a_i$, respectively. The adopted $\mathcal{L}_{dp}$ is stable during training, efficient to compute, and empirically well-performed.

\section{Empirical Study}
\label{sec:experiments}

In this section, we evaluate the performance of our FairSample framework for training demographic parity fair GCN classifiers and compare with the state-of-the-art baselines. 

\subsection{Datasets and Experiment Settings}

\subsubsection{Dataset}
We conducted experiments using a LinkedIn production dataset and five public fair GCN benchmark datasets~\cite{DBLP:conf/wsdm/DaiW21}. Table~\ref{tbl:datasets} shows some statistics of the datasets. 

\begin{table}[t]\smaller
\centering
\caption{Some statistics of the datasets. ``Intra. ratio'', ``Sens.'', and ``Nation.'' are short for ``Intra-group edge ratio'', ``Sensitive attribute'', and ``Nationality'', respectively.}
\begin{tabular}{|c|c|c|c|c|c|c|}
\hline
\textbf{Dataset} & 
NBA & 
PNL & 
PZL &
PNG &
PZG &
LNKD \\ \hline

\textbf{Nodes} &
403   & 
66,569 & 
67,797 & 
66,569 & 
67,797 & 
40,000 \\ \hline

\textbf{Edges} & 
16,570 & 
729,129 & 
882,765 & 
729,129 & 
882,765 & 
6,002,612  \\ \hline

\textbf{\makecell{Intra. \\ ratio}} & 
73\% & 
95\% & 
96\% & 
47\% & 
46\% & 
73\% \\ \hline

\textbf{\makecell{Sens.}} &
Nation. &
\makecell{Living\\ region} &
\makecell{Living\\ region} & 
Gender & 
Gender & 
\makecell{Level} 
\\ \hline
\textbf{\makecell{Task}} & 
\makecell{Income \\ level} &
Job & 
Job & 
Job & 
Job & 
Industry \\ \hline

\end{tabular}
\label{tbl:datasets}
\end{table}

The Pokec-n and Pokec-z datasets are subnetworks of a social network Pokec~\cite{takac2012data} in Slovakia, where the nodes are users and the edges are their connections. Pokec-n and Pokec-z are created by sampling the users from two provinces, Nitriansky and Zilinsky, respectively.
Each dataset consists of the users from the two major living regions of the corresponding province. The feature vector of a node is generated based on the corresponding user's profile, including living region, gender, spoken language, age, etc.
We use living region and gender as the sensitive attributes and conduct separate experiments for each. Both living region and gender are binary variables. The living region represents one of the two major living regions within each corresponding province. 
Denote by PNL and PZL the Pokec-n and Pokec-z datasets with living region as the sensitive attribute, respectively.
Similarly, we denote the Pokec-n and Pokec-z datasets with gender as the sensitive attribute as PNG and PZG, respectively. As shown in Table~\ref{tbl:datasets}, PNL and PZL datasets are homophilic, as there are more intra-group edges than inter-group edges in the graphs. Similarly, PNG and PZG datasets are heterophilic.
The classification task is to predict whether a user is a student. 

The \textbf{NBA} dataset consists of 403 NBA players as the nodes and their relationships on Twitter as the edges. The binary sensitive attribute of a player is whether a player is a U.S. or oversea player. The classification task is to predict whether the salary of a player is over the median. The feature vector of a node is generated based on the corresponding player's profile, including performance statistics, height, age, etc. The input graph is homophilic.

The \textbf{LNKD} dataset is a subset of the LinkedIn social networks where the nodes are  alumni from a US university, and the edges are the connections between them. The binary sensitive attribute of a user is the user's seniority level, high or low. The classification task is to predict whether a user is suitable for working in the ``Hospital \& Health Care'' industry. The feature vector of a node is generated based on the corresponding user's profile. The dataset is homophilic.

For the five large datasets, we follow the same setting used by Dai~\textit{et~al.}~\cite{DBLP:conf/wsdm/DaiW21}, where we randomly sample 500 labeled nodes as the training set, 25\% of the labeled nodes as the validation set, and 25\% of the labeled nodes as the test set. On the NBA dataset where the total number of nodes is less than 500, we randomly sample 50\% of the labeled nodes as the training set instead. The training, validation, and test sets have no overlaps. We generate $5$ sets of training, validation, and test datasets with different random seeds for each dataset.

\subsubsection{Baselines}
We compare the performance of FairSample (FS for short in the tables and figures) with 12 representative baselines belonging to three categories, namely, graph neural networks (GNNs) without fairness constraints, pre-processing fairness methods, and in-processing fairness methods.

\textbf{GAT}~\cite{DBLP:conf/iclr/VelickovicCCRLB18}, \textbf{DGI}~\cite{DBLP:conf/iclr/VelickovicFHLBH19}, and \textbf{GCA}~\cite{DBLP:conf/www/0001XYLWW21} are famous GNNs without fairness constraints, which only aim to optimize model utility.
GAT learns node classifiers in an end-to-end fashion. DGI and GCA employ contrastive objective functions to learn node embeddings in an unsupervised manner. 
We train a logistic regression on top of the embeddings learned by DGI and GCA, respectively, to predict node labels. 

We include three pre-processing fairness enhancing methods for GNNs. \textbf{FairAdj}~\cite{DBLP:conf/iclr/LiWZHL21} modifies the adjacency matrix of the input graph to enforce fairness constraints. A GNN is used to learn node embeddings on the modified adjacency matrix. 
\textbf{FCGE}~\cite{DBLP:conf/icml/BoseH19} is an adversarial framework to enforce fairness constraints on node embeddings. The embeddings are optimized to confuse a discriminator, which is jointly trained to predict the sensitive attribute value of a node from its embedding. FairAj and FCGE perform node classifications in the same way as DGI. \textbf{EDITS}~\cite{DBLP:conf/www/DongLJL22} mitigates the unfairness in GNNs by producing unbiased training data. EDITS debiases the input graph by altering its adjacency matrix and node attributes. We train a classic two-layer GCN~\cite{DBLP:conf/iclr/KipfW17} on top of the debiased graph to predict node labels.

Last, we include six in-processing fairness enhancing methods. \textbf{GraphSage}~\cite{DBLP:conf/nips/HamiltonYL17} and \textbf{PASS}~\cite{DBLP:conf/kdd/YoonGSNHY21} are celebrated sampling frameworks to efficiently train GCN classifiers. \textbf{Stratified-GraphSage} is an extension of GraphSage by performing stratified sampling within sensitive groups to sample nodes for embedding aggregation. It performs embedding aggregation using a more balanced neighborhood than GraphSage. The probability of sampling a neighbor $v_j$ of a node $v_i$ is set to $\text{P}(v_j | v_i) = \frac{1}{|\mathcal{A}| \cdot |\Gamma_{v_i} \cap \mathcal{V}_{s_j}|}$. 
To enforce fairness in GCN classifiers, the three methods train GCNs by adopting the same demographic parity regularizer as FairSample does. 
GSR, PASSR, and SGSR are short for GraphSage, PASS, and Stratified-GraphSage with the demographic parity regularizer, respectively. \textbf{FairGNN}~\cite{DBLP:conf/wsdm/DaiW21} is a state-of-the-art method for fair node classification tasks under the constraint of demographic parity. FairGNN adopts an adversarial framework~\cite{DBLP:conf/icml/BoseH19} to rectify the final layer representations in the target GNN classifier. 
We test a 2-layer GCN and a 2-layer GAT as the backbones of FairGNN, denoted by FGCN and FGAT, respectively.
\textbf{NIFTY}~\cite{DBLP:conf/uai/AgarwalLZ21} leverages contrastive learning to train GNNs under the constraint of counterfactual fairness. The method generates augmented counterfactual views of the input graph by flipping the node sensitive values. 
A 2-layer \textbf{multi-layer perceptron} (MLP)~\cite{goodfellow2016deep} predicts the node labels using only the node feature vectors. MLP adopts the same demographic parity regularizer as FairSample to enforce fairness.

\subsubsection{Implementation Details and Parameter Settings}

\label{sec:hyperparameter}

All GCN methods in our experiments train 2-layer GCNs with non-linear activation functions.
FairSample adopts a celebrated 2-layer GCN~\cite{DBLP:conf/iclr/KipfW17} as the pseudo label predictor. To be consistent, the dimensionalities of the node embeddings in all methods are set to 64.

\begin{table*}[t]
\centering
\smaller
\caption{The comparison between FairSample and the baselines. OOM is short for ``out of memory''. For the LNKD dataset, $x$ and $y$ are the anonymized mean accuracy and $\Delta DP$ of FairSample, respectively. Performance of the baselines on the LNKD dataset are presented in the difference with respect to FairSample.
FairSample outperforms all baselines up to 65.5\% in fairness improvement only at the cost of at most 5.0\% classification accuracy drop. The best records are in bold.} 
\begin{tabular}{|c|c@{\hspace{-3pt}}c|c@{\hspace{-3pt}}c|c@{\hspace{-3pt}}c|c@{\hspace{-3pt}}c|c@{\hspace{-3pt}}c|c@{\hspace{-3pt}}c|}
\hline
\multirow{2}{*}{Method} & 
\multicolumn{2}{c|}{NBA} &
\multicolumn{2}{c|}{PNL} &
\multicolumn{2}{c|}{PZL} &
\multicolumn{2}{c|}{PNG} &
\multicolumn{2}{c|}{PZG} &
\multicolumn{2}{c|}{LNKD} 
\\ \cline{2-13} 
 & \multicolumn{1}{l}{ACC (\%) } & $\triangle$DP (\%)  
 & \multicolumn{1}{l}{ACC (\%) } & $\triangle$DP (\%)  
 & \multicolumn{1}{l}{ACC (\%) } & $\triangle$DP (\%)  
 & \multicolumn{1}{l}{ACC (\%) } & $\triangle$DP (\%)  
 & \multicolumn{1}{l}{ACC (\%) } & $\triangle$DP (\%)  
 & \multicolumn{1}{l}{ACC (\%) } & $\triangle$DP (\%)  
\\ \hline
GAT & 
\multicolumn{1}{l}{\textbf{71.4$\pm$5.1}}& 11.9$\pm$8.9  &
\multicolumn{1}{l}{65.4$\pm$5.0} & 3.8$\pm$1.4 &
\multicolumn{1}{l}{66.8$\pm$4.7} & 7.6$\pm$4.3 &
\multicolumn{1}{l}{\textbf{68.0$\pm$1.0}} & 10.4$\pm$1.6 &
\multicolumn{1}{l}{65.1$\pm$4.9} & 5.3$\pm$2.4 &
\multicolumn{1}{l}{(x-1.9)$\pm$0.7} & (y+6.8)$\pm$5.5
\\ 
DGI &
\multicolumn{1}{l}{70.4$\pm$3.3} & 14.7$\pm$1.2 &
\multicolumn{1}{l}{67.1$\pm$2.1} & 2.4$\pm$1.3 &
\multicolumn{1}{l}{67.7$\pm$0.4} & 5.8$\pm$2.0 &
\multicolumn{1}{l}{65.6$\pm$1.8} & 11.4$\pm$3.2 &
\multicolumn{1}{l}{66.3$\pm$1.2} & 2.3$\pm$0.8 &
\multicolumn{1}{l}{(x-5.8)$\pm$0.8}         &
(y+3.1)$\pm$1.7 
\\ 
GCA & 
\multicolumn{1}{l}{68.1$\pm$4.7} & 14.2$\pm$6.1 & 
\multicolumn{1}{l}{65.3$\pm$1.5} & 2.7$\pm$2.1 & 
\multicolumn{1}{l}{67.0$\pm$1.0} & 4.7$\pm$2.0 & 
\multicolumn{1}{l}{66.8$\pm$1.0} & 9.6$\pm$2.4 &
\multicolumn{1}{l}{67.0$\pm$1.1} & 3.5$\pm$0.5 &
\multicolumn{1}{l}{OOM} & OOM 
\\ 
\hline
EDITS &
\multicolumn{1}{l}{60.0$\pm$5.9} & 17.4$\pm$10.0 & 
\multicolumn{1}{l}{OOM} & OOM & 
\multicolumn{1}{l}{OOM} & OOM & 
\multicolumn{1}{l}{OOM} & OOM & 
\multicolumn{1}{l}{OOM} & OOM & 
\multicolumn{1}{l}{OOM} & OOM \\ 
FairAdj & 
\multicolumn{1}{l}{70.4$\pm$7.5} & 10.4$\pm$5.6 & 
\multicolumn{1}{l}{OOM} & OOM &
\multicolumn{1}{l}{OOM} & OOM &
\multicolumn{1}{l}{OOM} & OOM &
\multicolumn{1}{l}{OOM} & OOM &
\multicolumn{1}{l}{OOM} & OOM \\ 
FCGE & 
\multicolumn{1}{l}{67.8$\pm$7.5} &
8.6$\pm$7.1 & 
\multicolumn{1}{l}{64.3$\pm$2.5} & 
7.1$\pm$1.5 & 
\multicolumn{1}{l}{67.3$\pm$0.6} & 
2.4$\pm$1.5 & 
\multicolumn{1}{l}{67.2$\pm$1.3} & 10.0$\pm$2.6 &
\multicolumn{1}{l}{66.9$\pm$0.7} & 2.6$\pm$1.3 &
\multicolumn{1}{l}{(x-10.1)$\pm$2.7} & (y+7.2)$\pm$5.7 
\\ 
\hline
MLP & 
\multicolumn{1}{l}{66.8$\pm$3.1} & 
13.4$\pm$10.7 & 
\multicolumn{1}{l}{64.3$\pm$1.4} & 
2.0$\pm$1.6 & 
\multicolumn{1}{l}{66.6$\pm$0.6} & 
3.0$\pm$1.9 &
\multicolumn{1}{l}{63.8$\pm$1.0} & 4.9$\pm$2.4 &
\multicolumn{1}{l}{65.4$\pm$1.6} & 2.9$\pm$1.6 &
\multicolumn{1}{l}{(x-0.4) $\pm$ 0.8} & 
(y+3.6)$\pm$3.3 
\\
GSR & 
\multicolumn{1}{l}{62.0$\pm$4.6} & 7.7$\pm$5.6 & 
\multicolumn{1}{l}{\textbf{67.7$\pm$1.1}} & 
1.1$\pm$0.5 & 
\multicolumn{1}{l}{\textbf{68.0$\pm$0.5}} & 
2.1$\pm$0.9 &
\multicolumn{1}{l}{67.0$\pm$1.0} & 7.5$\pm$3.0 &
\multicolumn{1}{l}{66.5$\pm$0.7} & 1.7$\pm$1.0 &
\multicolumn{1}{l}{(x-7.2)$\pm$2.5} &
(y+4.9)$\pm$2.6 
\\ 
SGSR & 
\multicolumn{1}{l}{56.2$\pm$2.8} & 
8.6$\pm$9.3 & 
\multicolumn{1}{l}{66.3$\pm$1.3} & 
1.4$\pm$0.7 & 
\multicolumn{1}{l}{67.5$\pm$0.5} & 
2.3$\pm$1.5 & 
\multicolumn{1}{l}{65.5$\pm$0.9} & 11.0$\pm$1.4 &
\multicolumn{1}{l}{65.7$\pm$0.8} & 3.0$\pm$1.3 &
\multicolumn{1}{l}{(x-7.1)$\pm$2.6} & 
(y+1.9)$\pm$1.8 
\\ 
PASSR & 
\multicolumn{1}{l}{66.1$\pm$2.2} & 
6.5$\pm$6.6 & 
\multicolumn{1}{l}{67.3$\pm$1.1} & 
1.1$\pm$0.6 &
\multicolumn{1}{l}{67.6$\pm$0.3} & 
1.5$\pm$0.9 & 
\multicolumn{1}{l}{67.0$\pm$0.4} & 10.7$\pm$1.8 &
\multicolumn{1}{l}{\textbf{67.1$\pm$0.7}} & 2.2$\pm$1.3 &
\multicolumn{1}{l}{(x-0.2)$\pm$1.9} & 
(y+2.0)$\pm$1.9 
\\ 
FGCN & 
\multicolumn{1}{l}{67.8$\pm$4.7} & 
11.7$\pm$2.2 & 
\multicolumn{1}{l}{66.2$\pm$1.3} & 
2.5$\pm$1.1 & 
\multicolumn{1}{l}{66.3$\pm$1.3} & 
2.7$\pm$1.2 & 
\multicolumn{1}{l}{66.6$\pm$0.7} & 5.3$\pm$2.3 &
\multicolumn{1}{l}{66.5$\pm$0.8} & 2.9$\pm$1.1 &
\multicolumn{1}{l}{(x-2.2)$\pm$5.1} & 
(y+3.8)$\pm$2.8 \\ 
FGAT & 
\multicolumn{1}{l}{66.1$\pm$5.6} & 
16.8$\pm$11.3 & 
\multicolumn{1}{l}{67.1$\pm$1.5} & 
2.4$\pm$0.8 & 
\multicolumn{1}{l}{65.3$\pm$1.8} & 
3.1$\pm$1.8 & 
\multicolumn{1}{l}{66.9$\pm$0.7} & 6.2$\pm$4.2 &
\multicolumn{1}{l}{66.7$\pm$2.0} & 2.6$\pm$1.6 &
\multicolumn{1}{l}{(x-5.5)$\pm$2.9} & 
(y+7.6)$\pm$4.7 \\ 
NIFTY & 
\multicolumn{1}{l}{59.0$\pm$9.9} & 
7.8$\pm$8.8 & 
\multicolumn{1}{l}{64.7$\pm$2.1} & 
6.5$\pm$1.0 & 
\multicolumn{1}{l}{66.3$\pm$1.1} & 
5.4$\pm$1.5 & 
\multicolumn{1}{l}{65.5$\pm$1.6} & 13.3$\pm$0.8 &
\multicolumn{1}{l}{63.9$\pm$2.6} & 1.7$\pm$0.6 &
\multicolumn{1}{l}{(x-8.9)$\pm$1.3} & 
(y+11.1)$\pm$5.3 
\\ 
\hline
\hline
FS & 
\multicolumn{1}{l}{67.8$\pm$4.3} &
\textbf{6.2$\pm$6.9} & 
\multicolumn{1}{l}{67.4$\pm$0.6} & 
\textbf{0.9$\pm$0.5} & 
\multicolumn{1}{l}{67.7$\pm$0.7} & 
1.1$\pm$0.9 & 
\multicolumn{1}{l}{67.2$\pm$0.6} & \textbf{4.7$\pm$2.1} &
\multicolumn{1}{l}{\textbf{67.1$\pm$0.5}} & \textbf{1.4$\pm$1.3} &
\multicolumn{1}{l}{x$\pm$0.6} & 
\textbf{y$\pm$0.9} 
\\
FS(noise) & 
\multicolumn{1}{l}{67.3$\pm$3.4} & 
10.1$\pm$2.6 & 
\multicolumn{1}{l}{ 67.3$\pm$0.7} & 
1.0$\pm$0.7 & 
\multicolumn{1}{l}{67.5$\pm$1.1} & 
\textbf{0.9$\pm$0.7} & 
\multicolumn{1}{l}{66.8$\pm$0.6} & 5.2$\pm$2.0 &
\multicolumn{1}{l}{67.0$\pm$0.6} & 2.3$\pm$1.4 &
\multicolumn{1}{l}{\textbf{(x+0.2)$\pm$1.5}} & 
(y+1.4)$\pm$2.1 
\\ \hline
\end{tabular}
\label{tbl:baselines}
\end{table*}

For the NBA dataset, we sample $k=10$ neighbors per node in the four sampling-based methods, FairSample, GSR, SGSR, and PASSR. This is because, empirically, we observe that the accuracies of the learned node classifiers are saturated when $k=10$. For the same reason, we set $k=6$ for the LinkedIn dataset and $k=20$ for the remaining four large datasets. All the other baselines are allowed to use the full neighborhood to train GNNs.
We test a wide range of the demographic parity regularizer weight $\alpha \in \{0, 1, 2, 5, 10\}$ in the loss functions of the sampling-based methods.
We vary the hyperparameters of the other fairness enhancing methods within the ranges suggested by their authors to maximize their performance.
We train the sampling-based methods for 300 epochs with an early stopping rule on validation loss. For the other baselines, we use the suggested number of training epochs and early stopping rules in the released source codes.
In FairSample, we tune the number of injected edges per node $m \in \{0, 4, \ldots, 20\}$ and fix the confidence threshold of the pseudo label predictor $\tau=0.8$.
FairGNN assumes that the sensitive values of some nodes may be missing and employs an estimator to predict the missing values. In our experiments, we observe that the worst accuracy of the sensitive value estimator of FairGNN across all datasets is 78\%. 
To fairly compare with FairGNN, we test a variant FairSample (Noise) of FairSample. We uniformly randomly flip 22\% of the sensitive  values in each dataset when FairSample (Noise) is run.

We use the published Python codes of GAT~\cite{FairGNN}, DGI~\cite{DGI}, GCA~\cite{GCA}, FairAdj~\cite{FairAdj}, FCGE~\cite{FCGE}, FairGNN~\cite{FairGNN}, EDITS~\cite{EDITS}, and NIFTY~\cite{NIFTY}. The remaining algorithms are implemented in Python. All experiments are conducted on a server with 8 Tesla P100 GPUs, and 672G main memory.


\begin{table*}[t]
\centering
\caption{The $p$-values of the hypothesis tests. The $p$-values that are less than the significance level $1\%$ are in bold.}
\begin{tabular}{|c|c|c|c|c|c|c|c|c|c|c|c|c|c|}
\hline
Method    & GAT & DGI    & GCA   & MLP   & EDITS & FairAdj & FCGE        & GSR   & SGSR  & PASSR & FGCN & FGAT & NIFTY  \\ \hline
ACC $p$-value &
0.64 &
0.98 &
0.94 &
0.99 &
0.90 &
0.98 &
0.99 &
0.99 &
0.99 &
0.91 &
0.99 &
0.99 &
0.99 
\\ \hline
$\Delta$DP $p$-value  &
\textbf{$\mathbf{8e^{-6}}$} &
\textbf{$\mathbf{6e^{-7}}$} & 
\textbf{$\mathbf{5e^{-6}}$} & 
\textbf{$\mathbf{7e^{-4}}$} & 
0.156 & 
\textbf{$\mathbf{1e^{-3}}$} &
\textbf{$\mathbf{3e^{-7}}$} &
\textbf{$\mathbf{7e^{-4}}$} &
\textbf{$\mathbf{3e^{-4}}$} &
\textbf{$\mathbf{4e^{-4}}$} &
\textbf{$\mathbf{2e^{-4}}$} &
\textbf{$\mathbf{4e^{-4}}$} &
\textbf{$\mathbf{1e^{-5}}$}
\\ \hline

\end{tabular}
\label{tbl:p_value}
\end{table*}

\subsection{Can FairSample Learn Fair and Accurate GCNs?}
\label{sec:exp_baselines}

We evaluate the performance, in demographic parity and accuracy, of FairSample in learning GCNs.  We use accuracy (ACC) to evaluate the utility of a learned node classifier. 
Following the best practice in literature~\cite{DBLP:conf/icml/AgarwalBD0W18, DBLP:conf/wsdm/DaiW21}, the violation of demographic parity can be measured by
$\Delta DP = \text{max}_{a_{i_1}, a_{i_2} \in \mathcal{A}}(| \mathbb{E}[f(\mathbf{x})|s=a_{i_1}] - \mathbb{E}[f(\mathbf{x})|s=a_{i_2}]|)$. A classifier with a smaller $\Delta DP$ is more fair.

Following Donini~\textit{et~al.}~\cite{DBLP:conf/nips/DoniniOBSP18}, we apply a two-step mechanism to choose the best hyperparameter setting for each method.
In the first step, we run each method with each hyperparameter setting 5 times. The highest mean validation accuracy $\text{Acc}_{best}$ across all the methods and hyperparameter settings is identified. A validation accuracy threshold is set to $\text{Acc}_{t} = \text{Acc}_{best} * 0.95$.
In the second step, for each method, among the hyperparameter settings with mean validation accuracies higher than $\text{Acc}_{t}$, we select the hyperparameter setting with the smallest mean validation $\Delta DP$. If a method cannot learn a model with a mean validation  accuracy passing $\text{Acc}_{t}$, we pick the hyperparameter setting with the best mean validation  accuracy.

The performance of all methods is reported in Table~\ref{tbl:baselines}. 
GAT, DGI, and GCA have poor fairness performance, as they only aim to optimize the accuracy of GNNs. They do not achieve the best accuracy on all datasets. This is because the fairness constraints on some baselines may act as regularizers, which potentially reduce overfitting, and thus empirically improve generalization performance~\cite{DBLP:conf/aies/IslamPF21, chu2021fedfair}. GCA does not scale up well as it uses the whole $K$-hop neighborhoods of each node to train node embeddings, and hence suffers from high memory footprints. FairSample uses a computation graph sampler with a small number of parameters to down-sample the input graph. Thus, it enjoys better scalability than GCA.

The pre-processing methods, EDITS, FairAdj, and FCGE do not have good performance in either accuracy or fairness. As they learn node embeddings in an unsupervised manner, they may not achieve a good performance in our node classification tasks.
FairAdj and EDITS cannot handle large datasets. As they adjust the adjacency matrix by assigning a learnable parameter to each edge of the input graph, these methods have a large memory consumption. 
FairSample has better performance than the pre-processing methods, as it can explicitly optimize the utility and fairness in training the node classifiers. 

MLP has worse accuracy and fairness than FairSample, since the graph structure information can help improve the accuracy in node classification~\cite{DBLP:conf/wsdm/DaiW21}. Moreover, FairSample is more fair than MLP, as FairSample intelligently aggregates feature vectors from nodes in different sensitive groups to mitigate unfairness in predictions.

NIFTY has the worst fairness performance among the in-processing baselines, as the method aims to optimize GNNs under the counterfactual fairness rather than the demographic parity fairness.
Despite using a sampled neighborhood, FairSample achieves better accuracy than FairGNN.
FairGNN also does not have good fairness performance, as FairGNN only tries to improve fairness by rectifying the parameters of GNNs and does not mitigate the bias in graph structures. 
To the contrary, by rectifying both the model parameters and the graph structures, FairSample achieves an outstanding performance of training fair GCNs with good accuracies.
The performance of FairSample (Noise) is better than FairGNN, which indicates that FairSample is robust under sensitive value noise. 
FairSample constantly outperforms PASSR and SGSR in both accuracy and fairness.
This is because their sampling strategies do not take both utility and fairness into consideration.

FairSample constantly outperforms GSR in fairness with just a slight accuracy loss on the PNL and PZL datasets. The reason is that our hyperparameter selection method~\cite{DBLP:conf/nips/DoniniOBSP18} favors fairness. Among all the hyperparameter settings with accuracy above a certain threshold, we choose the one with the best fairness performance. To conduct a more comprehensive comparison between FairSample and GSR, we vary the hyperparameter settings of each method, and find the Pareto frontier~\cite{Ngatchou2006-yz} between accuracy and $\Delta DP$. The Pareto frontier of a method characterizes its optimal tradeoff between accuracy and fairness~\cite{DBLP:journals/corr/abs-2201-00292}. 
A model has a better fairness-utility tradeoff if the model achieves the same fairness level with better utility, and vice versa~\cite{DBLP:conf/iclr/LiWZHL21}.
Figure~\ref{fig:pareto_frontier} shows the mean and the standard deviation of $1 - \Delta DP$ and the accuracy of each hyperparameter setting in the Pareto frontier of each method on the PNL and PZL datasets.
Each point in the figure represents the performance of a method with a specific hyperparameter setting.
We observe that FairSample is closer to the top right corner than GSR, indicating that FairSample achieves a better fairness-utility tradeoff. 
For the same level of fairness, FairSample can learn more accurate GCNs than GSR.

To study the statistical significance, for each baseline, we conduct the one sided Wilcoxon signed-rank test~\cite{woolson2007wilcoxon} between FairSample and the baseline to test if FairSample is more fair. The null hypothesis is that the $\Delta DP$ of FairSample is greater than or equal to that of the baseline method. Similarly, we also conduct one sided Wilcoxon signed-rank tests to test if FairSample is less accurate than the baselines. The $p$-values of the hypothesis tests are reported in Table~\ref{tbl:p_value}. 
It shows that FairSample is more fair than all but one baselines with statistical significance. Although the hypothesis test between the $\Delta DP$ of FairSample and EDITS is not statistically significant, as shown in Table~\ref{tbl:baselines}, FairSample has substantially better scalability than EDITS. Table~\ref{tbl:p_value} also shows that there is no statistical significance showing that FairSample is less accurate than any baselines. 

\begin{figure}[t]
\centering
\subfloat[PNL]{
\includegraphics[page=1, width=0.45\linewidth, scale=0.3]{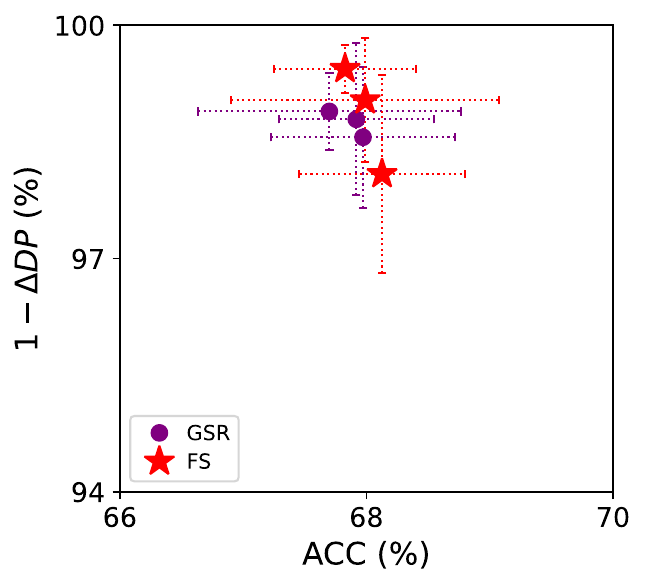}
\label{figure:baseline_pokec_z}
}
\subfloat[PZL]{
\includegraphics[page=2, width=0.45\linewidth, scale=0.3]{Figure/pareto_curve.pdf}
\label{figure:baseline_pokec_n}
}
\caption{The demographic parity and accuracy tradeoff of the models trained by FairSample and GSR. A point closer to the top-right corner is better.
}
\label{fig:pareto_frontier}
\end{figure}

The results clearly demonstrate that FairSample can learn fair and accurate GCNs and outperform the baselines.

\subsection{Can FairSample Mitigate Graph Structure Bias and Improve Fairness?}
\label{exp:structure_only}
As explained in Section~\ref{sec:analyze}, mitigating the structure bias of an input graph may improve the demographic parity of the GCNs built on graphs. To verify our analysis, we evaluate the capability of FairSample in training fair GCNs by reducing graph structure bias alone.

\begin{table}[t]
\caption{The comparison of FairSample with the baselines in training fair GCNs without demographic parity regularizers in the loss functions. The results of all methods on the LNKD dataset are reported in the same way as in Table~\ref{tbl:baselines}. The best records are in bold.}  
\begin{adjustbox}{max width=.98\linewidth}
\begin{tabular}{|c|l|l|l|l|l|}
\hline
Dataset              & Metrics & GSR & SGSR & PASSR & FS$^{nr}$ \\ \hline
\multirow{3}{*}{NBA} 
& ACC (\%)  & 61.5$\pm$2.6 & 56.2$\pm$2.8 & 66.1$\pm$2.2 & \textbf{66.3$\pm$4.1} \\\cline{2-6}
& $\triangle$DP (\%)  & 5.3$\pm$4.5  & 8.6$\pm$9.3 & 6.5$\pm$6.6 & \textbf{5.1$\pm$3.3}     \\\cline{2-6}
\hline
\hline
\multirow{3}{*}{PNL} 
& ACC (\%)  & 67.9$\pm$0.7 &  67.9$\pm$0.8  & \textbf{68.0$\pm$0.6} &  \textbf{68.0$\pm$1.1}   \\\cline{2-6}
& $\triangle$DP (\%)  & 1.9$\pm$0.9 & 1.5$\pm$0.9 & 2.0$\pm$1.4 & \textbf{1.0$\pm$0.8}     \\\cline{2-6}
\hline
\hline
\multirow{3}{*}{PZL} 
& ACC (\%)  & 67.8$\pm$0.5 & \textbf{68.5$\pm$0.6} & \textbf{68.5$\pm$0.3} & 68.1$\pm$0.3     \\\cline{2-6}
& $\triangle$DP (\%)  & 6.6$\pm$2.1 & 6.6$\pm$2.2 & 6.2$\pm$1.4 & \textbf{5.7$\pm$1.1}     \\\cline{2-6}
\hline
\hline
\multirow{3}{*}{PNG} 
& ACC (\%)  & 66.5$\pm$0.7 &  65.5$\pm$0.9  & 67.0$\pm$0.4 &  \textbf{67.5$\pm$0.8}   \\\cline{2-6}
& $\triangle$DP (\%)  & 11.1$\pm$2.3 & 11.0$\pm$1.4 & 10.7$\pm$1.8 & \textbf{10.0$\pm$2.3}   \\\cline{2-6}
\hline
\hline
\multirow{3}{*}{PZG}
& ACC (\%)  & 66.4$\pm$0.9 & 65.5$\pm$1.5 & 67.2$\pm$1.0 & \textbf{67.6$\pm$0.6}  \\\cline{2-6}
& $\triangle$DP (\%)  & 2.4$\pm$2.1 & 2.8$\pm$1.2 & 1.8$\pm$1.6 & \textbf{1.6$\pm$1.0}  \\\cline{2-6}
\hline
\hline
\multirow{3}{*}{LNKD}
& ACC (\%)  & (x-7.8)$\pm$1.0 & (x-9.1)$\pm$1.1 & (x-0.6)$\pm$1.9 & \textbf{x$\pm$1.0}     \\\cline{2-6}
& $\triangle$DP (\%)  & (y+5.7)$\pm$3.3 & (y+6.5)$\pm$6.5 & (y+1.0)$\pm$1.9 & \textbf{y$\pm$1.2}\\\cline{2-6}
\hline
\end{tabular}
\end{adjustbox}
\label{tbl:effect_graph_structure_bias}
\end{table}

To demonstrate the effectiveness of our graph structure bias reduction strategy, we disable the demographic parity regularizer in FairSample and train GCNs with the cross-entropy loss function. Denote by FairSample$^{nr}$ this variant of FairSample. We compare FairSample with three sampling baselines. The weight of the fairness regularizer for each of the three sampling baselines, GSR, SGSR, and PASSR, is set to $0$.
The mean and standard deviation of the performance of all methods are reported in Table~\ref{tbl:effect_graph_structure_bias}.

GS performs poorly because it uniformly samples neighbors without fairness consideration and results in discrimination between different sensitive groups due to inheriting the graph structure bias. SGS sacrifices considerable model accuracy in order to achieve fairness. 
PASS achieves the best accuracy but not fairness score among the baselines, as its sampling policy is optimized for model utility but not fairness. FairSample$^{nr}$ employs a graph enrichment strategy and a fairness-aware computation graph sampling strategy to reduce graph structure bias and improve fairness. Therefore, it achieves superior performance in training fair  and accurate GCNs.

The experimental results strongly confirm that FairSample can mitigate graph structure bias and improve fairness.

\subsection{Convergence and Computational Cost}
\label{exp:convergence}

\begin{figure}[t]
\centering
\includegraphics[height=5mm, page=5, trim={10 0 10 0}, clip]{Figure/POKEC_Z_converge.pdf}

 
\subfloat[The accuracy convergence curves of FairSample, FGAT, and NIFTY]{
\includegraphics[page=1, width=0.45\linewidth, scale=0.3]{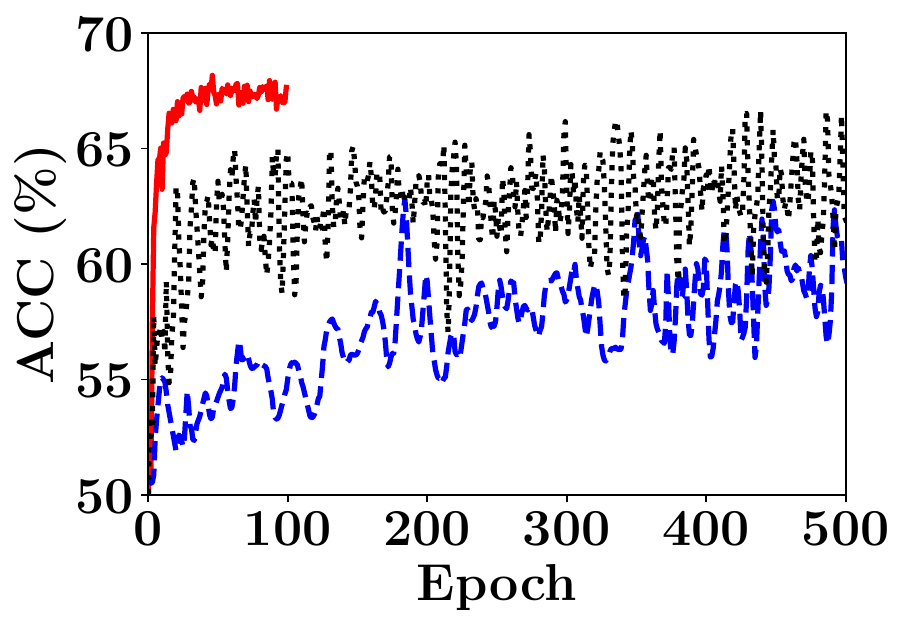}
\label{figure:acc_convergence_pokec_z_fgnn}
}
\hspace{2mm}
\subfloat[The $\Delta DP$ convergence curves of FairSample, FGAT, and NIFTY]{
\includegraphics[page=3, width=0.45\linewidth, scale=0.3]{Figure/POKEC_Z_converge.pdf}
\label{figure:dp_convergence_pokec_z_fgnn}
}

\subfloat[The accuracy convergence curves of FairSample, GSR, and PASSR]{
\includegraphics[page=2, width=0.45\linewidth, scale=0.3]{Figure/POKEC_Z_converge.pdf}
\label{figure:acc_convergence_pokec_z_pass}
}
\hspace{2mm}
\subfloat[The $\Delta DP$ convergence curves of FairSample, GSR, and PASSR]{
\includegraphics[page=4, width=0.45\linewidth, scale=0.3]{Figure/POKEC_Z_converge.pdf}
\label{figure:dp_convergence_pokec_z_pass}
}
\qquad
\caption{The accuracy and $\Delta DP$ convergence curves of FairSample, FGAT, NIFTY, GSR, and PASSR across training epochs. 
For the sake of display clarity, we report the results of PASSR and GSR separately from the other baselines in Figures~\ref{figure:acc_convergence_pokec_z_pass} and~\ref{figure:dp_convergence_pokec_z_pass}.}
\label{fig:convergence}
\end{figure}

We analyze the convergence and computational cost of FairSample and the baselines in training fair GCNs.

Figure~\ref{fig:convergence} shows the convergence curves of FairSample, PASSR, GSR, FGAT, and NIFTY in accuracy and $\Delta DP$ on the PZL dataset.
Similar observations are obtained on the other datasets. 
We train each method 5 times and report the mean accuracy and mean $\Delta DP$ on the test datasets across the training epochs.
FairSample, GSR, and PASSR converge and stop after around 100 training epochs. 
FGAT and NIFTY require more epochs to converge, we report their convergence curves across the first 500 training epochs.

FGAT has poor convergence performance in both accuracy and $\Delta DP$, as it employs adversarial learning to train fair GNNs, which is known to be unstable during training~\cite{DBLP:conf/aies/BeutelCDQWLKBC19, DBLP:conf/wsdm/DaiW21}. 
NIFTY generates counterfactual graph views by uniformly removing edges and node attributes, as well as flipping sensitive values. However, removing some influential edges and node attributes may deteriorate the utility of GCN classifiers~\cite{DBLP:journals/corr/abs-2201-08549}.
Figure~\ref{figure:dp_convergence_pokec_z_fgnn} shows that NIFTY has the worst $\Delta DP$. This is because NIFTY is not explicitly designed to optimize the demographic parity of GCNs.

Figures~\ref{figure:acc_convergence_pokec_z_pass} and~\ref{figure:dp_convergence_pokec_z_pass} show that FairSample converges very quickly with only around 50 training epochs. 
FairSample outperforms PASSR and GSR in training fair GCNs with a clear margin. This is because FairSample modifies the computation graphs of GCNs to rectify graph structure bias. 

Figures~\ref{figure:train_time_pokec_z} and~\ref{figure:gpu_memory} show the average training time and the average GPU memory usage of FairSample and the other GNN baselines, respectively. It is worth noting that the two baselines, EDITS and FairAdj, exhibit high GPU memory consumption, which leads to the ``out of memory" issue. Therefore, these two methods are not included in the figure.
Limited by space, we only report the results on the PZG dataset. Similar observations are obtained on the other datasets. 

Clearly, FairSample is more efficient and consumes less GPU memory than the non-sampling baselines, as FairSample leverages sampling strategies to reduce the sizes of computation graphs. FairSample requires longer training time and GPU memory than SGSR and GSR. 
The training time of SGSR, GSR, and FairSample is 21.09, 21.56, and 32.27 seconds, respectively. The GPU memory costs of SGSR, GSR, and FairSample are 1371, 1371, and 1413 MB, respectively.
The two baseline methods employ heuristic strategies for computation graph down-sampling, while FairSample uses a learnable smart sampling strategy. Therefore, FairSample has higher computational cost. 

The experimental results confirm that FairSample brings only a small computation overhead to train fair GCNs.

\begin{figure}[t]
\centering
\subfloat[Training time]{
\includegraphics[page=1, width=0.45\linewidth, scale=0.3]{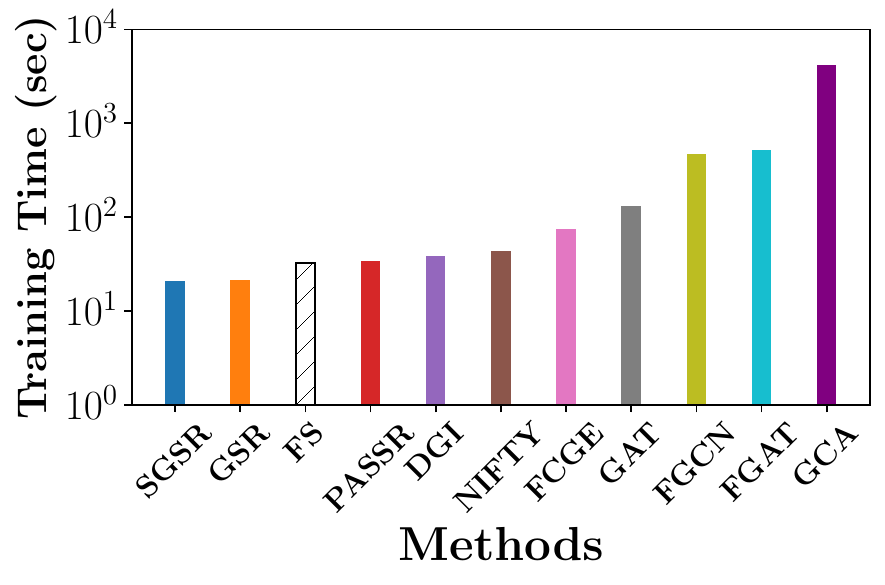}
\label{figure:train_time_pokec_z}
}
\subfloat[GPU Memeory Usage]{
\includegraphics[page=1, width=0.45\linewidth, scale=0.3]{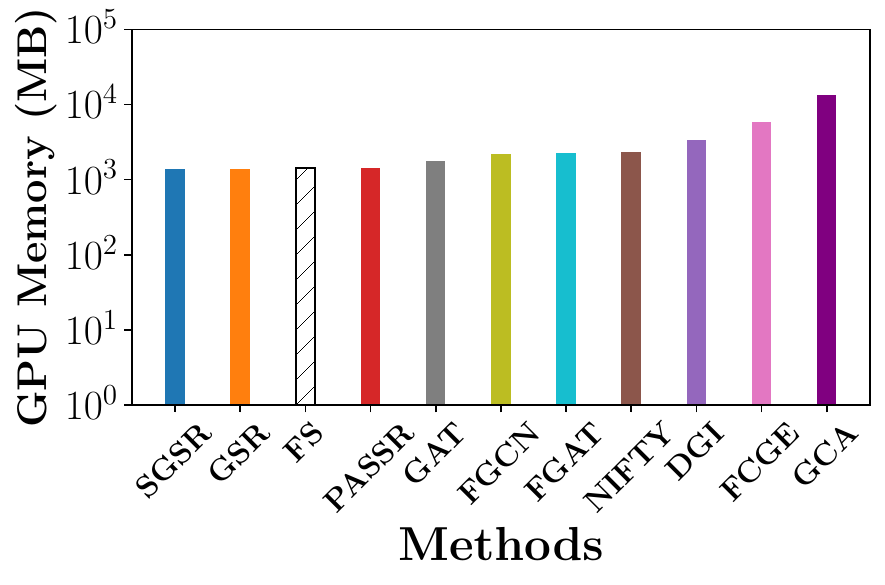}
\label{figure:gpu_memory}
}
\caption{The training time (in second) and GPU memory usage (in MB) of FairSample and the baselines on the PZG dataset. The bars in the figures are sorted in ascending order from left to right based on their values. The y-axis is in logarithmic scale.}
\label{fig:train_time}
\end{figure}

\subsection{Ablation Study}
\label{exp:ablation}

\begin{table}[t]
\caption{Comparison between FairSample and variants. The results of all methods on the LNKD dataset are reported in the same way as in Table~\ref{tbl:baselines}.  The best records are in bold.}
\begin{adjustbox}{max width=.98\linewidth}
\begin{tabular}{|c|l|l|l|l|l|}
\hline
Dataset              & Metrics & FS$^{ne}$ & FS$^{ns}$ & FS$^{nr}$ & FS \\ \hline
\multirow{2}{*}{NBA} 
& ACC (\%)  & 66.8$\pm$2.2 & 65.6$\pm$6.2 &  65.6$\pm$6.2 & \textbf{67.8$\pm$4.3}     \\\cline{2-6}
& $\triangle$DP (\%)  & \textbf{5.1$\pm$5.7} & 8.3$\pm$3.8 & 5.7$\pm$5.3 & 6.2$\pm$6.9    \\\cline{2-6}
\hline
\hline
\multirow{2}{*}{PNL} 
& ACC (\%)  & 67.4$\pm$0.7 & 67.7$\pm$1.1 & \textbf{68.0$\pm$1.1} & 67.4$\pm$0.7     \\\cline{2-6}
& $\triangle$DP (\%)  &  \textbf{0.9$\pm$0.5} & 1.1$\pm$0.5 & 1.0$\pm$0.8 & \textbf{0.9$\pm$0.5}     \\\cline{2-6}
\hline
\hline
\multirow{2}{*}{PZL} 
& ACC (\%)  & 67.5$\pm$0.5 & 68.0$\pm$0.5 & \textbf{68.1$\pm$0.3} & 67.7$\pm$0.7   \\\cline{2-6}
& $\triangle$DP (\%)  & 1.5$\pm$0.7 & 2.1$\pm$0.9 & 5.5$\pm$1.1 & \textbf{1.1$\pm$0.9}     \\\cline{2-6}
\hline
\hline
\multirow{2}{*}{PNG} 
& ACC (\%)  & 67.2$\pm$0.6 & 67.0$\pm$1.0 & \textbf{67.5$\pm$0.8} & 67.2$\pm$0.6 \\\cline{2-6}
& $\triangle$DP (\%)  &  \textbf{4.7$\pm$2.1} & 7.5$\pm$3.0 & 10.0$\pm$2.3 & \textbf{4.7$\pm$2.1}   \\\cline{2-6}
\hline
\hline
\multirow{2}{*}{PZG} 
& ACC (\%)  & 66.9$\pm$1.0 & 66.4$\pm$0.9 & \textbf{67.6$\pm$0.6} & 67.1$\pm$0.5 \\\cline{2-6}
& $\triangle$DP (\%)  & 1.7$\pm$1.4 & 2.4$\pm$2.1 & 1.6$\pm$1.0 & \textbf{1.4$\pm$1.0}    \\\cline{2-6}
\hline
\hline
\multirow{2}{*}{LNKD} 
& ACC (\%)  & (x-0.3)$\pm$1.5 & (x-7.2)$\pm$2.5 & (x-0.1)$\pm$1.5 & \textbf{x$\pm$0.6}     \\\cline{2-6}
& $\triangle$DP (\%)  & (y+0.2)$\pm$0.8 & (y+4.9)$\pm$2.6 & (y+0.9)$\pm$1.3 & \textbf{y$\pm$0.9}     \\\cline{2-6}
\hline
\end{tabular}
\end{adjustbox}
\label{tbl:abalation_study} 
\end{table}

We conduct an ablation study to understand the effectiveness of the edge injector $f_{E}$, the computation graph sampler $f_{S}$, and the demographic parity regularizer in FairSample.

We implement an ablation for each of the components. Denote by FairSample$^{ne}$ the ablation for our edge injector, which trains GCN classifiers on the original input graph. Denote by FairSample$^{nr}$ the ablation for the demographic parity regularizer in the loss function, which trains GCNs by optimizing the cross-entropy loss function. Denote by FairSample$^{ns}$ the ablation for our computation graph sampler, which trains GCNs by performing uniform random node sampling.

Table~\ref{tbl:abalation_study} shows the ablation performance. On all the datasets except the smallest one (NBA), the fairness of FairSample is better than or equal to that of FairSample$^{ne}$. This result aligns well with our analysis that adding inter-group (intra-group) edges to connect similar nodes from the same class benefits the fairness of the GCNs on homophilic (heterophilic) graphs. On the NBA datasets, FairSample underperforms slightly in fairness compared to FairSample$^{ne}$. We select the hyperparameters for each method based on the fairness of the validation dataset. Given FairSample's higher model complexity, it may not generalize as well as FairSample$^{ne}$ on the NBA test datasets.

Similar to FairSample$^{ne}$, the baseline method PASSR also uses reinforcement learning to train its sampling policy.
Different from PASSR, FairSample$^{ne}$ adopts the idea of stratified sampling to train fair GCN classifiers. The performance of PASSR shown in Tables~\ref{tbl:baselines} and the performance of FairSample$^{ne}$ shown in Table~\ref{tbl:abalation_study} indicate that FairSample$^{ne}$ outperforms PASSR in fairness on all cases.

FairSample consistently outperforms FairSample$^{ns}$ with respect to both fairness. FairSample$^{ns}$ uniformly samples the neighbors of each node to train GCNs.  Thus, it is less effective in picking a set of informative and balanced neighbors for GCN training. 

Comparing the performance of FairSample$^{nr}$ and FairSample, we can see that adding a demographic parity regularizer in the loss functions can improve model fairness. Due to the high model complexity and high nonlinearity of GCNs, the fairness regularizer may regularize the GCNs to improve their testing accuracies~\cite{chu2021fedfair}.

\subsection{Parameter Sensitivity}
\begin{figure}[t]
\centering
\subfloat[\# Injected edges per node $m$]{
\includegraphics[page=1, width=0.45\linewidth, scale=0.3]{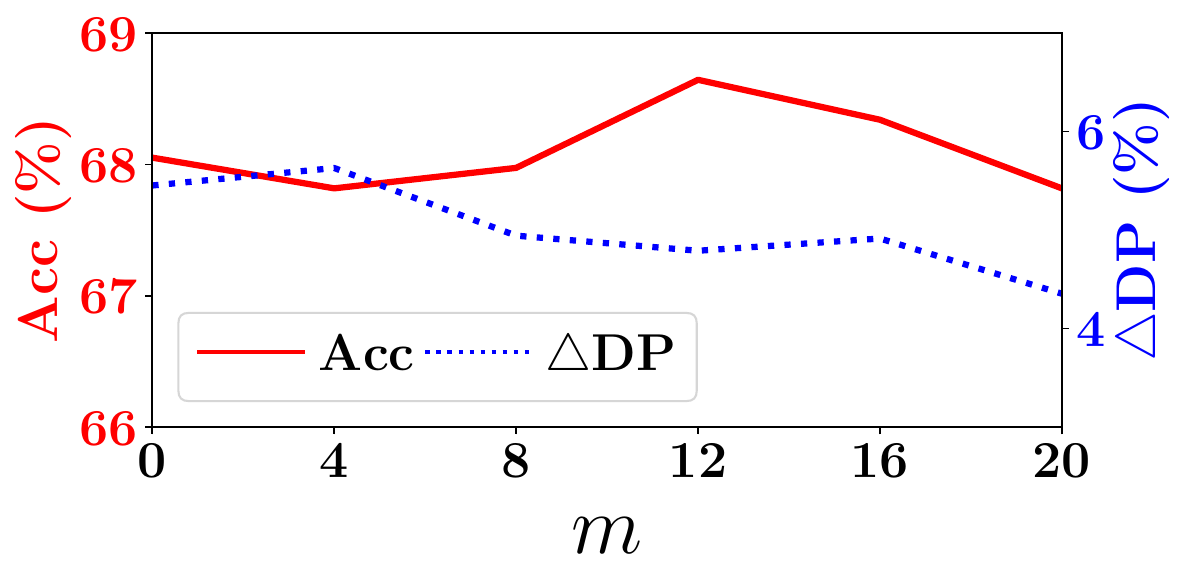}
\label{figure:m}
}
\hspace{2mm}
\subfloat[Fairness regularizer weight $\alpha$]{
\includegraphics[page=2, width=0.45\linewidth, scale=0.3]{Figure/sensitive.pdf}
\label{figure:alpha}
}

\subfloat[Depth of GCN $K$]{
\includegraphics[page=3, width=0.45\linewidth, scale=0.3]{Figure/sensitive.pdf}
\label{figure:K}
}
\hspace{2mm}
\subfloat[\# Sampled child node per node $k$]{
\includegraphics[page=4, width=0.45\linewidth, scale=0.3]{Figure/sensitive.pdf}
\label{figure:k}
}
\caption{The effect of the hyperparameters $m$, $\alpha$, $K$, and $k$ on the demographic parity and accuracy of FairSample.}
\label{fig:parameter_sensitivity}
\end{figure}

We study how the number of injected edges per node $m$, the weight of the demographic parity regularizer $\alpha$, the depth of GCNs $K$, and the number of sampled child nodes per node $k$ affect the demographic parity and accuracy of FairSample. 

To study the effect of a particular hyperparameter on GCNs trained by FairSample, we fix the other hyperparameters and vary the hyperparameter we are interested in.
For each hyperparameter setting, we conduct the experiments 5 times, and report the mean accuracy and $\Delta DP$. Limited by space, we only report the results on the PZL dataset in Figure~\ref{fig:parameter_sensitivity}. Similar observations are obtained on the other datasets. 

Figure~\ref{figure:m} indicates that injecting inter-group edges can reduce $\Delta DP$. This is because injecting inter-group edges can help our computation graph sampler in sampling a balanced neighborhood for training fair GCNs.
As shown in Figure~\ref{figure:m}, the accuracies of the GCNs change within a small range as $m$ increases, which suggests that the injected edges can improve model fairness and retain a good model accuracy. 
Figure~\ref{figure:alpha} shows that a larger $\alpha$ leads to a more fair model, but incurs a higher accuracy cost. Figure~\ref{figure:K} shows the effect of the depth of GCNs $K$ on model accuracy and fairness. When $K=0$, the GCNs degenerate to logistic regression models. On the one hand, the accuracies of the GCNs increase as $K$ increases from $0$ to $2$, as the graph structure information is used to learn better representations.
On the other hand, a large $K$ may cause the over-smooth issue~\cite{DBLP:conf/icml/ChenWHDL20, DBLP:conf/kdd/LiuGJ20}, and hence does not help with model accuracy. The figure also indicates that a large $K$ can benefit mode fairness. This is because deeper GCNs are capable of learning more sophisticated decision boundaries to satisfy the demographic parity constraint. 
Since GCNs use the $K$-hop neighbors to train node embeddings, the computational cost increases with $K$. The effect of $k$ is reported in Figure~\ref{figure:k}. Sampling more neighbors for GCN training can learn GCNs with better utility. However, it also leads to higher computation cost.
When $k=1$, the GCNs make approximately random predictions. Therefore, they have small $\Delta DP$.
When the GCN classifiers have reasonable accuracies, the $\Delta DP$ of GCNs decrease as $k$ increases. This is because each node is capable of aggregating information from more diversified neighbors.

\section{Conclusions}
\label{sec:conclusion}

In this paper, we tackle the challenging problem of building demographic parity fair and accurate GCNs efficiently for semi-supervised node classification tasks.
We systematically analyze how graph structure bias, node attribute bias, and model parameters may affect the demographic parity of GCNs. Based on the theoretical analysis, we develop a novel framework FairSample for GCN training, which employs a graph enrichment strategy and a computation graph sampling strategy.  We report extensive experiments demonstrating the superior capability of FairSample in training GCN classifiers with good demographic parity and high accuracy.  In five public benchmark datasets and one LinkedIn production dataset, we observe an improvement up to 65.5\% in fairness at the cost of only at most 5\% classification accuracy drop.

FairSample is the first step into the frontier of designing sampling strategies for fair and accurate GCNs. There are a few interesting ideas for future work. First, 
it is valuable to extend FairSample to improve the other fairness notions in GCNs, such as individual fairness~\cite{DBLP:conf/innovations/DworkHPRZ12}. Second, an in-depth theoretical understanding of the potential and boundaries of sampling strategies for fair GCNs in more general settings remains an interesting challenge.

\nop{
\ifCLASSOPTIONcompsoc
  \section*{Acknowledgments}
\else
  \section*{Acknowledgment}
\fi

The authors would like to thank...
}

\bibliographystyle{IEEEtran}
\bibliography{ref}

\end{document}